\documentclass[11pt]{article}
\usepackage[margin=1in]{geometry}
\usepackage{amsmath,amssymb,amsthm,mathtools}
\usepackage{booktabs}
\usepackage{pgfplots}
\pgfplotsset{compat=1.17}
\usepackage[hidelinks]{hyperref}
\usepackage{enumitem}

\newtheorem{theorem}{Theorem}
\newtheorem{proposition}[theorem]{Proposition}
\newtheorem{corollary}[theorem]{Corollary}

\theoremstyle{definition}
\newtheorem{example}[theorem]{Example}
\newtheorem{remark}[theorem]{Remark}

\newcommand{\R}{\mathbb{R}}
\newcommand{\E}{\mathbb{E}}
\newcommand{\Xh}{\widehat{X}}
\newcommand{\Wh}{\widehat{W}}
\newcommand{\Hh}{\widehat{H}}
\newcommand{\fro}[1]{\lVert #1\rVert_F}
\newcommand{\ip}[2]{\langle #1,#2\rangle_F}
\newcommand{\base}{\mathrm{BASE}}
\newcommand{\qep}{\mathrm{QEP}}
\DeclareMathOperator{\clip}{clip}

\title{When Does Compensating Inherited Quantization Error Help?\\
Mean-Square Guarantees for Layer-Wise Quantization and Their Limits}
\author{Tianhao Guo\thanks{Tianhao Guo is with the College of Physics and Electronic Engineering, Shanxi University, Taiyuan 030006, China (e-mail: \href{mailto:tianhao_guo@sxu.edu.cn}{tianhao\_guo@sxu.edu.cn}). ORCID: \href{https://orcid.org/0000-0002-1317-868X}{0000-0002-1317-868X}.\newline Code: \url{https://github.com/tellsgo/quantization-compensation}.}}
\date{12 October 2026}

\begin{document}
\maketitle

\begin{abstract}
Layer-wise post-training quantization methods let each layer compensate for the quantization error inherited from earlier layers. We ask under what assumptions such compensation guarantees a network output error no larger than without it. For linear networks of any depth, with the required calibration Gram matrices invertible, if the rounding error of each layer has zero conditional mean and a deterministic second-moment map common to both schemes, as under subtractive dither on an unbounded grid, then compensation with strengths in $[0,1]$ fixed in advance gives a mean-square output error no larger than no compensation, each scheme being followed along its own trajectory. The exact gain separates error removed directly from reduced excitation of later rounding noise. Outside this setting the comparison can fail. With rectified linear activations it fails for two independent reasons: noise changes activation states, and, even when no state changes, the correction moves error from samples on which a unit is inactive to samples on which it is active. With stable activations, commutators between the calibration projection and the activation patterns bound the excess. With a bounded codebook, even unique exact layer-wise minimizers with no overload can give a larger output error. For fixed rounding errors it already fails for two scalar linear layers, refuting a published universal guarantee. An exact residual identity gives the necessary and sufficient condition on a common layer input. All these comparisons hold the calibration input fixed. On two 8B models at 4 bits, fixed half-strength compensation beat uncompensated rounding in every model, text and draw tested, yet no rule for choosing the strength beat a layer-wise rule consistently, and lowering the full-network loss on the data used to choose did not reliably lower it on new text.
\end{abstract}

\section{Introduction}

Layer-wise post-training quantization processes the linear maps of a network one at a time, choosing for each a quantized weight matrix that reproduces the layer output on calibration data \cite{gptq,awq,quip,quipsharp}. Because every layer is treated separately, the error made in early layers reaches later layers through their inputs. Several recent methods change the layer objective so that a layer matches the output of the \emph{original} network on the \emph{original} input, while being fed the quantized input; the layer then compensates for the inherited error. The same objective appears under the names quantization error propagation \cite{qep}, asymmetric calibration \cite{gptaq} and activation drift correction \cite{watersic}, and in \cite{qronos,nestquant}.

These methods share the motivation of reducing errors inherited from earlier quantized layers. This raises a stronger question: under what assumptions does layer-wise compensation guarantee that the network output error is no larger than without compensation? The answer is not obvious. Each layer solves a least-squares problem, and the solution is closer to the target than the uncorrected one \emph{before} rounding. After rounding this need not be so, and a comparison on a common input says nothing yet about two schemes that, from the second layer on, see different inputs. That is the question of this paper. The comparison requires us to follow both the inherited error and the activation that multiplies the next rounding error. Reducing the first is not enough; the second must be controlled along the two trajectories as well. Throughout, the calibration input is fixed and the randomness is that of rounding; whether an improvement on calibration data carries over to other inputs is a separate question, on which Section~\ref{sec:empirical} reports experiments.

Related analyses exist. SPFQ \cite{spfq} gives high-probability full-network error bounds for a stochastic path-following quantizer. The analysis of OPTQ and Qronos \cite{qronostheory} gives layer-wise bounds for both algorithms and compares their error decompositions and upper bounds. These results do not establish an ordering of the realized total errors of a compensated and an uncompensated scheme.

We establish a network-level comparison under explicit noise assumptions, give the gain exactly, and determine where the comparison fails. The fixed-error guarantee of \cite{qep} is one such failure. Its mechanism is simple: the inherited error and the new rounding error can cancel, and removing the first leaves the second without its counterpart. Whether compensation helps is therefore a question about the correlation between the two errors, and the positive results below are obtained by controlling it. We use the formulation of \cite{qep} to study a family of compensated layer objectives parameterized by the strength. The related methods motivate the analysis; their quantizers, solvers and regularization choices require separate treatment. QEP denotes the compensated scheme and BASE the uncompensated one in the model studied here.

\paragraph{Contributions.}
\begin{enumerate}[leftmargin=*,itemsep=2pt]
\item \textbf{A comparison along two trajectories, and the gain it yields} (Section~\ref{sec:network}). For linear networks of any depth, with rectangular layers allowed, Theorem~\ref{thm:linear} gives mean-square output error no larger with compensation than without it, for strengths in $[0,1]$ fixed in advance, under the required Gram condition, zero conditional mean and a deterministic second-moment map common to both schemes. The proof carries two moments through the layers, that of the error and that of the quantized activation; the second governs how much the next rounding error is amplified. Proposition~\ref{prop:gain} separates the gain into error removed directly and reduced excitation of later noise. These compare compensation with no compensation. Ordering two non-zero strength schedules requires additional row-space conditions (Corollary~\ref{cor:sched}); the closed form of Example~\ref{ex:iso} gives a ratio of expected errors in the specified random-teacher model, not a prediction for every fixed network.
\item \textbf{Two obstructions for ReLU networks, and their geometric cost} (Section~\ref{sec:relu}). The comparison fails under the same dither when noise changes activation states (Example~\ref{ex:relu}), and also when no state changes (Example~\ref{ex:noflip}). In the latter case the correction mixes calibration samples across activation patterns. Proposition~\ref{prop:commute} characterizes when this mixing is harmless for all inherited errors. Under stable activations, Theorem~\ref{thm:relu} bounds the excess by a term controlled by commutators of the calibration projection with the sample groups; Corollary~\ref{cor:hr} gives its small-error order.
\item \textbf{The local condition and the limits of the network comparison} (Sections~\ref{sec:onelayer}, \ref{sec:counter} and~\ref{sec:actual}). An exact residual identity gives the necessary and sufficient condition for compensation to help on a common input. Its expectation covers subtractive dither with fixed inter-column feedback, without giving a network guarantee for GPTQ or QuIP. Fixed perturbations can instead cancel inherited error, invalidating the published universal guarantee and strength monotonicity. For an actual four-level codebook, an open family has a larger output error despite unique exact layer-wise minimizers and interior targets, with no overload (Theorem~\ref{thm:actual}). Synthetic computations illustrate the positive results, and a scoped Lean formalization provides a complementary check (Sections~\ref{sec:numerics} and~\ref{sec:lean}). Experiments on two small and two 8B models address the question the theorems leave open, whether gains on calibration data carry over to new inputs (Section~\ref{sec:empirical}).
\end{enumerate}

The paper goes from one layer to the network: the per-layer identity and its expectation (Section~\ref{sec:onelayer}), linear networks (Section~\ref{sec:network}), ReLU networks (Section~\ref{sec:relu}), then fixed rounding errors and the published guarantee (Section~\ref{sec:counter}) and actual quantizers (Section~\ref{sec:actual}).

\paragraph{Scope.} The tools are standard: orthogonal projections, dithered rounding and order relations between second-moment matrices. The contribution is the comparison they yield and its limits. We propose no new quantizer and do not dispute reported empirical gains. We have not established stable activations in trained networks. The experiments of Section~\ref{sec:empirical} do not satisfy our hypotheses and are not tests of them. The projection contraction, continuous optimum and each scheme's own bound in \cite{qep} remain valid; Section~\ref{sec:counter} isolates the comparisons that fail.

\section{Setting}\label{sec:setting}

Consider a feedforward network, with the notation of \cite[Appendix~B.3]{qep},
\[
X_1=X,\qquad X_{l+1}=\sigma_l(W_lX_l),\qquad l=1,\dots,L,
\]
where each $\sigma_l$ is $\gamma_l$-Lipschitz in the Frobenius norm and $\sigma_l(0)=0$. A quantized forward pass uses weights $\Wh_l$: $\Xh_1=X$, $\Xh_{l+1}=\sigma_l(\Wh_l\Xh_l)$. Write $\delta_l=X_l-\Xh_l$ and $\Hh_l=\Xh_l\Xh_l^\top$. We require $\Hh_l$ to be invertible only at layers of positive strength: a layer with $\alpha_l=0$ is not corrected, its target is $W_l$, and no Gram matrix is inverted there. The weights may be rectangular; invertibility of quantized weights is an additional assumption of the published fixed-error guarantee examined in Section~\ref{sec:counter}, not of the positive results.

The matrices $E_l$ represent rounding errors. Sections~\ref{sec:network} and~\ref{sec:relu} impose conditional moment assumptions on random errors; Section~\ref{sec:counter} instead fixes the same arbitrary matrices in both schemes. The uncorrected scheme uses $\Wh_l^{\base}=W_l+E_l$. QEP uses
\[
W_l^*(\alpha_l)=W_l+\alpha_lW_l\delta_l\Xh_l^\top\Hh_l^{-1},\qquad
\Wh_l^{\qep}=W_l^*(\alpha_l)+E_l,\qquad \alpha_l\in[0,1].
\]
For $M\in\{\base,\qep\}$ define the residuals and the bound
\[
R_l^M=W_lX_l-\Wh_l^M\Xh_l^M,\qquad
U^M=\sum_{k=1}^{L}\Bigl(\prod_{s=k+1}^{L}\gamma_s\lVert W_s\rVert_2\Bigr)\gamma_k\fro{R_k^M}.
\]

\paragraph{Summary of the results.} Several different comparisons arise, and their answers differ according to whether they are read exactly, to first order in a fixed shared perturbation, or in mean square for errors with zero conditional mean and a common second-moment map, of which subtractive dither is the main instance. Table~\ref{tab:map} records what this paper establishes; ``open'' means not settled here.

\begin{table}[h]
\centering\small
\begin{tabular}{@{}p{4.3cm}p{3.2cm}p{3.2cm}p{4.0cm}@{}}
\toprule
Statement & Exact, shared $E$ & First order, shared $E$ & Mean square, zero conditional mean\\
\midrule
$U^{\qep}\le U^{\base}$ & fails (Thm.~\ref{thm:counter}) & fails (Thm.~\ref{thm:counter}) & fails for some error laws allowed by Thm.~\ref{thm:linear} (Sec.~\ref{sec:network}); open for uniform dither\\
Output error of QEP $\le$ that of BASE & fails (Thm.~\ref{thm:counter}); for deterministic codebooks see Thm.~\ref{thm:actual} & fails (Thm.~\ref{thm:counter}) & linear networks: holds (Thm.~\ref{thm:linear}). ReLU: fails in general (Ex.~\ref{ex:relu}, \ref{ex:noflip}); holds in the commuting case of Thm.~\ref{thm:relu}, otherwise up to the excess bounded there\\
$\fro{\delta_{L+1}^{\qep}}\le U^{\base}$ & fails (Ex.~\ref{ex:composite}) & holds, in the repaired form of App.~\ref{app:more} & holds for linear networks (Thm.~\ref{thm:linear}) and in the commuting case of Thm.~\ref{thm:relu}; otherwise open\\
Output error non-increasing in $\alpha$ & fails (Thm.~\ref{thm:alpha}) & fails (Thm.~\ref{thm:alpha}) & holds for two linear scalar layers (Prop.~\ref{prop:dither}); fails between non-zero schedules for sign-valued errors (Sec.~\ref{sec:network})\\
Layer residual on a common input, QEP $\le$ BASE & for $\alpha>0$, iff $2c\le(2-\alpha)d$ (Cor.~\ref{cor:iff}) & same condition & holds for fixed-step dither on an unbounded grid (Thm.~\ref{thm:meansq}); can fail under overload (Sec.~\ref{sec:dither})\\
\bottomrule
\end{tabular}
\caption{What holds and what fails. The first four rows compare the two schemes on their own trajectories; the last row is a per-layer comparison on a common input. In the last column the bounds are compared through $\E(U^M)^2$, and entries established only for sign-valued errors say so. The network comparisons require the common deterministic second-moment map as well as zero conditional mean, and the additional hypotheses of the cited results.}\label{tab:map}
\end{table}

\section{One layer: when the correction helps}\label{sec:onelayer}

We begin with a single layer on a common input. An exact identity decides whether the correction helps for a given rounding error; taking expectations gives the conditions under which it helps on average, and subtractive dither supplies them. The network-level results of Sections~\ref{sec:network} and~\ref{sec:relu} build on this section.

\subsection{An exact identity}\label{sec:identity}

Fix one layer and one quantized input $\Xh$, and drop the layer index. Write
\[
A=W\delta,\qquad B=E\Xh,\qquad P=\Xh^\top\Hh^{-1}\Xh,\qquad R_\alpha=A(I-\alpha P)-B .
\]
Thus $R_0$ is the BASE residual on this input and $R_\alpha$ the QEP residual with strength $\alpha$. Let
\[
d=\fro{AP}^2,\qquad c=\ip{AP}{B}.
\]

\begin{proposition}[Residual identity]\label{prop:identity}
Let $A,B$ be real matrices of the same shape and $P$ a symmetric idempotent matrix. For all real $\alpha,\beta$,
\begin{align}
\fro{R_\alpha}^2-\fro{R_0}^2&=-\alpha(2-\alpha)\,d+2\alpha\,c,\label{eq:gap}\\
\fro{R_\beta}^2-\fro{R_\alpha}^2&=(\beta-\alpha)\bigl[(\alpha+\beta-2)\,d+2c\bigr].\label{eq:pair}
\end{align}
\end{proposition}

\begin{proof}
$R_\alpha=R_0-\alpha AP$, so $\fro{R_\alpha}^2=\fro{R_0}^2-2\alpha\ip{R_0}{AP}+\alpha^2d$. Because $P$ is a symmetric idempotent, $\ip{A}{AP}=\fro{AP}^2=d$, hence $\ip{R_0}{AP}=d-c$. This gives \eqref{eq:gap}; \eqref{eq:pair} follows by subtraction.
\end{proof}

\begin{corollary}\label{cor:iff}
With the notation above:
\begin{enumerate}[label=(\roman*),leftmargin=*,itemsep=2pt]
\item For $\alpha>0$: $\fro{R_\alpha}\le\fro{R_0}$ if and only if $2c\le(2-\alpha)d$.
\item $\alpha\mapsto\fro{R_\alpha}$ is non-increasing on $[0,1]$ if and only if $c\le0$.
\item $\fro{R_\alpha}\le\fro{R_0}$ for all $\alpha\in[0,1]$ if and only if $2c\le d$.
\item If $d>0$, the minimizer of $\fro{R_\alpha}$ over $[0,1]$ is $\alpha^\star=\clip(1-c/d,\,0,\,1)$.
\end{enumerate}
\end{corollary}

\begin{proof}
(i) is \eqref{eq:gap} divided by $\alpha>0$. For (ii), if $c\le0$ the bracket in \eqref{eq:pair} is non-positive whenever $\alpha+\beta\le2$. Conversely, if $c>0$ then $AP\neq0$, so $d>0$; take $h=\min(1,c/d)$, and then $\fro{R_1}^2-\fro{R_{1-h}}^2=h(2c-hd)>0$. (iii) follows from (i) since $(2-\alpha)d\ge d$ on $[0,1]$, with equality at $\alpha=1$. (iv): $\fro{R_\alpha}^2$ is a convex quadratic in $\alpha$ with vertex $1-c/d$.
\end{proof}

When $2c\le d$ but $c>0$, full correction is no worse than no correction, yet the residual first decreases and then increases as $\alpha\to1$. ``No worse at $\alpha=1$'' and ``monotone in $\alpha$'' are different properties.

\paragraph{Geometric reading.} At $\alpha=1$ the condition $2c\le d$ is equivalent to $\fro{B-AP}\ge\fro{B}$. For $B=E\Xh$ we have $BP=B$, so $B$ and $AP$ lie in the same subspace, and the component $A(I-P)$ is common to both residuals. The correction is harmful exactly when the new perturbation term $B$ is closer to the inherited term $AP$ than to the origin, that is, when the two would largely cancel. In the counterexample of Theorem~\ref{thm:counter} in Section~\ref{sec:counter}, $AP=-\varepsilon$ and $B=-\varepsilon-\varepsilon^2$.

\begin{remark}[Scope of $\alpha^\star$]\label{rem:alphastar}
The formula for $\alpha^\star$ holds $B$ fixed. In an actual quantizer $B$ depends on $\alpha$ through the rounding of $W^*(\alpha)$ (Section~\ref{sec:actual}), so $\alpha^\star$ is not by itself a rule for choosing the propagation strength.
\end{remark}

\subsection{In expectation}\label{sec:meansq}

\begin{theorem}[Per-layer mean-square statement]\label{thm:meansq}
Let $A,B,P$ be random, with $P$ almost surely a symmetric idempotent and $A,B$ square integrable. Put $D=\E\fro{AP}^2$ and $C=\E\ip{AP}{B}$. For every fixed $\alpha$,
\[
\E\fro{R_\alpha}^2-\E\fro{R_0}^2=-\alpha(2-\alpha)D+2\alpha C .
\]
In particular, if $\E[B\mid A,P]=0$ then $C=0$ and $\E\fro{R_\alpha}^2=\E\fro{R_0}^2-\alpha(2-\alpha)D\le\E\fro{R_0}^2$ for $\alpha\in[0,2]$, with $\E\fro{R_\alpha}^2$ non-increasing on $[0,1]$.
\end{theorem}

\begin{proof}
Take expectations in \eqref{eq:gap}. Under the conditional-mean assumption, $C=\E\,\ip{AP}{\E[B\mid A,P]}=0$.
\end{proof}

\begin{remark}[A zero marginal mean is not enough]
Let all quantities be scalars, $P=1$ and $A=B=\pm1$ with equal probability. Then $\E B=0$, but $R_0=0$ and $R_\alpha=-\alpha A$, so $\E R_\alpha^2=\alpha^2>0=\E R_0^2$ for every $\alpha\neq0$.
\end{remark}

\begin{remark}[The perturbation may depend on $\alpha$]\label{rem:energy}
If the perturbation is $B_\alpha$ and $\E[B_\alpha\mid A,P]=0$ for each $\alpha$, the same computation gives
\[
\E\fro{R_\alpha}^2-\E\fro{R_0}^2=-\alpha(2-\alpha)D+\E\fro{B_\alpha}^2-\E\fro{B_0}^2 .
\]
Zero conditional mean alone therefore does not suffice. Improvement holds exactly when $\E\fro{B_\alpha}^2-\E\fro{B_0}^2\le\alpha(2-\alpha)D$; a perturbation energy that does not grow with the correction is a sufficient condition.
\end{remark}

\subsection{Subtractive dither supplies both conditions}\label{sec:dither}

\paragraph{Scope.} The statements in this subsection concern rounding to an unbounded uniform grid with a step size fixed in advance. They apply to a bounded codebook only when no target falls outside its range; when targets do fall outside, the per-layer statement can fail, as shown at the end of the subsection. Section~\ref{sec:empirical} reports how often corrected weights leave the range of the original row on two small models; we have no estimate beyond those, and we do not attribute the unfavourable 2-bit cases reported in \cite{qep} to overload.

Let $Q_\Delta$ denote rounding to the nearest point of $\Delta\mathbb{Z}$. The subtractively dithered quantizer maps $t$ to $Q_\Delta(t+u)-u$, where $u$ is uniform on $[-\Delta/2,\Delta/2)$, independent of $t$, and known to the decoder. Its error $Q_\Delta(t+u)-u-t$ is uniform on $[-\Delta/2,\Delta/2)$ and independent of $t$ \cite{schuchman,graystockham}; the vector analogue for lattices is in \cite{zamirfeder}. This holds for an unbounded grid; a bounded codebook requires in addition that no overload occurs.

Suppose each entry of $W_l^*(\alpha_l)$ is quantized in this way, with mutually independent fresh dithers that are independent of all earlier layers, and a step size that is fixed in advance and does not depend on the correction strength. Then $E_l$ has independent entries, uniform with variance $\sigma^2=\Delta^2/12$, and is independent of $(W_l\delta_l,\Xh_l)$, which are functions of earlier dither only. Hence $\E[E_l\Xh_l\mid W_l\delta_l,\Xh_l]=0$, and $\E\fro{E_l\Xh_l}^2=\sigma^2n_l\,\E\fro{\Xh_l}^2$ does not depend on $\alpha_l$, where $n_l$ is the number of rows of $W_l$. Both conditions of Theorem~\ref{thm:meansq} and Remark~\ref{rem:energy} hold, and on a common input
\[
\E\fro{R_l^{\qep}}^2=\E\fro{\widetilde R_l^{\base}}^2-\alpha_l(2-\alpha_l)\,\E\fro{W_l\delta_lP_l}^2 .
\]

This is a statement about one layer on a common input. It does not by itself give a network-level comparison, because the two schemes produce different inputs for later layers. Section~\ref{sec:network} treats the network level; we first compute a two-layer example exactly.

\begin{proposition}[Two-layer linear network with dither]\label{prop:dither}
Let $L=2$, all matrices $1\times1$, identity activations, teacher weights $W_1,W_2$, input $x$, and subtractively dithered quantization with error variance $\sigma^2$ in both layers, with $\Xh_2\neq0$ almost surely. With $\alpha_1$ arbitrary and $\alpha_2=\alpha$, the mean-square output error of QEP is
\[
x^2\sigma^2\bigl[(1-\alpha)^2W_2^2+W_1^2+\sigma^2\bigr].
\]
It is non-increasing in $\alpha$ on $[0,1]$, and strictly decreasing when $xW_2\neq0$; $\alpha=0$ gives the uncorrected scheme, and $\alpha=1$ removes the contribution $x^2\sigma^2W_2^2$ of the first-layer error.
\end{proposition}

\begin{proof}
Let $r_1,r_2$ be the quantization errors. They are independent, zero mean, with variance $\sigma^2$, and $r_2$ is independent of $r_1$ whatever the second-layer target is. The first-layer correction vanishes, so $\Xh_2=(W_1+r_1)x$. The second-layer target is $W_2\bigl(1-\alpha r_1/(W_1+r_1)\bigr)$, hence the output is $W_2W_1x+(1-\alpha)W_2r_1x+r_2(W_1+r_1)x$. The output error is $(1-\alpha)W_2r_1x+r_2(W_1+r_1)x$; its two terms are uncorrelated and $\E[r_2^2(W_1+r_1)^2]=\sigma^2(W_1^2+\sigma^2)$.
\end{proof}

In a Monte Carlo check with $\Delta=1$, two parameter settings, five values of $\alpha$ and $3\times10^5$ samples each, the estimates differ from this formula by at most $0.0017$. This is a sanity check, not a proof.

\paragraph{Rounding with feedback.} Methods such as GPTQ and QuIP round the entries of a weight row in sequence, adding to each target a fixed linear combination of the errors already made: $z_j=w_j+\sum_{i<j}(w_i-q_i)U_{ij}$, with $U$ strictly upper triangular and determined by $\Xh$ \cite{quip,quipsharp}. Suppose each $q_j$ is obtained from $z_j$ by subtractive dither on an unbounded grid, with fresh independent dither and fixed step sizes. Each innovation $n_j=q_j-z_j$ is then uniform and independent of the past, including its own target $z_j$, and the row error $e=q-w$ satisfies $e(I+U)=n$, that is, $e=n(I+U)^{-1}$. Hence $\E[E\Xh\mid W\delta,\Xh]=0$, and $\E\fro{E\Xh}^2$ does not depend on $\alpha$ provided $U$, the column order and the step sizes are not re-estimated when $\alpha$ changes. The per-layer statement therefore extends to dithered rounding with fixed feedback; the entries of $E$ are now correlated, which the statement does not require them not to be. It does not extend to ordinary stochastic rounding, whose error variance depends on the position of the target within its cell.

\paragraph{Bounded codebooks.} For a bounded uniform codebook $\{a,a+\Delta,\dots,b\}$ with nearest-point saturation, the dithered reconstruction of a target $t$ is distributed as $\clip(t,a,b)+V$ with $V$ uniform on $[-\Delta/2,\Delta/2)$. Overload thus introduces a bias $\clip(t,a,b)-t$ and leaves the variance unchanged, and the bias can break the per-layer statement. Take the codebook $\{0,1,2,3\}$ with $\Delta=1$ and
\[
Z=\begin{pmatrix}1&0\\1&\tfrac12\end{pmatrix},\qquad
W=\begin{pmatrix}4&2\\0&1\end{pmatrix},\qquad
T=\begin{pmatrix}4&1\\0&1\end{pmatrix},\qquad X=W^{-1}TZ ,
\]
so that $\Xh=Z$, $P=I$ and $W^*(\alpha)=W+\alpha(T-W)$. With entrywise dithered rounding,
\[
\E\fro{WX-\Wh(\alpha)Z}^2=\tfrac58-\tfrac{\alpha}{2}+\tfrac54\alpha^2 ,
\]
which equals $\tfrac58$ at $\alpha=0$ and $\tfrac{11}{8}$ at $\alpha=1$; the minimum is at $\alpha=\tfrac15$. The entry $4$ is clipped to $3$ at every $\alpha$, and correcting the neighbouring entry removes a term that partly cancelled this bias through the correlated inputs. Overload of the uncorrected targets is not needed: with the same $T$, $W=\bigl(\begin{smallmatrix}2&2\\0&1\end{smallmatrix}\bigr)$, $Z=\bigl(\begin{smallmatrix}1&0\\2&1/2\end{smallmatrix}\bigr)$ and $X=W^{-1}TZ$, every entry of $W$ lies in the codebook, only the corrected entry $2+2\alpha$ overloads, for $\alpha>\tfrac12$, and the loss $\tfrac78+\tfrac14(1-\alpha)^2+\max(0,2\alpha-1)^2$ falls to $\tfrac{15}{16}$ at $\alpha=\tfrac12$ and then rises to $\tfrac{15}{8}$ at $\alpha=1$, above its value $\tfrac98$ at $\alpha=0$. These examples use entrywise rounding with a non-diagonal input Gram matrix; they do not contradict Remark~\ref{rem:local}, which concerns exact minimization over the codebook.

\paragraph{Cost.} The reconstruction $Q_\Delta(t+u)-u$ is not a codebook point. The decoder must regenerate the dither, for instance from a seed, and subtract it. There is also an accuracy cost. On an unbounded grid the dithered error has mean square $\Delta^2/12$ whatever the input. For inputs that vary smoothly on the scale of $\Delta$, deterministic rounding has the same mean-square error; on a coarse grid, deterministic rounding of concentrated inputs can have a smaller one. We have not measured either cost. The diagnostic of Section~\ref{sec:empirical} includes dithered rounding of single layers; we have not evaluated dithered propagation end to end on any trained model; Section~\ref{sec:numerics} does so on small synthetic networks.

\section{Linear networks}\label{sec:network}

The per-layer statements compare the two schemes on a common input. In a network the two schemes produce different inputs from the second layer on, and each must be followed along its own trajectory. For linear networks the mean-square comparison holds for the whole network.

\begin{theorem}[Linear networks]\label{thm:linear}
Let all activations be the identity and let the strengths $\alpha_l\in[0,1]$ be fixed in advance. Let $\mathcal F_l$ denote the history before layer $l$ is quantized. Suppose that, for both schemes and every $l$,
\[
\E[E_l\mid\mathcal F_l]=0,\qquad \E[E_lTE_l^\top\mid\mathcal F_l]=\Psi_l(T)
\]
for every bounded $\mathcal F_l$-measurable positive semidefinite $T$, where $\Psi_l$ is a deterministic positive linear map, the same for both schemes. The network, the input $X_1$ and the depth are fixed. Suppose also that, for every $l$ with $\alpha_l>0$, the Gram matrix $\Hh_l$ along the QEP trajectory is almost surely invertible. Write $C_l^M=\E\bigl[\delta_l^M(\delta_l^M)^\top\bigr]$ and $G_l^M=\E\bigl[\Xh_l^M(\Xh_l^M)^\top\bigr]$. Then these moments are finite, and for every $l$,
\[
C_l^{\qep}\preceq C_l^{\base},\qquad G_l^{\qep}\preceq G_l^{\base},
\]
and in particular $\E\fro{\delta_{L+1}^{\qep}}^2\le\E\fro{\delta_{L+1}^{\base}}^2$.
\end{theorem}

Entrywise subtractive dither on an unbounded grid, with independent entries and a step fixed in advance, satisfies the hypotheses with $\Psi_l(T)=\sigma_l^2\,\mathrm{tr}(T)\,I$; with entry-dependent steps, $\Psi_l$ is the corresponding fixed weighted trace map. The weight matrices may be rectangular, subject to the rank restriction discussed after the proof.

\begin{proof}
Fix a layer of the QEP trajectory and write $Z=\Xh_l$, $X=X_l$, $P=P_l$, $W=W_l$, $\alpha=\alpha_l$. If $\alpha=0$ the correction is skipped, $T=Z$, and the two identities below hold trivially. Otherwise, since $ZP=Z$, the next pre-activation is $\Wh Z=WT+E_lZ$ with $T=(1-\alpha)Z+\alpha XP$, and $T-X=(Z-X)(I-\alpha P)$. Two pointwise identities follow from $P=P^\top=P^2$:
\begin{align}
(Z-X)(Z-X)^\top-(T-X)(T-X)^\top&=\alpha(2-\alpha)\,(Z-X)P(Z-X)^\top\succeq0,\label{eq:psd1}\\
(1-\alpha)ZZ^\top+\alpha XX^\top-TT^\top&=\alpha\,X(I-P)X^\top+\alpha(1-\alpha)(Z-XP)(Z-XP)^\top\succeq0.\label{eq:psd2}
\end{align}
The second-moment relation extends from bounded $T$ to every $\mathcal F_l$-measurable positive semidefinite $T$, as an identity in $[0,\infty]$, by truncating $T$ on $\{\lVert T\rVert\le n\}$ and monotone convergence. Taking the trace in \eqref{eq:psd2} gives $\fro{T}^2\le(1-\alpha)\fro{Z}^2+\alpha\fro{X}^2$, and with $\kappa_l=\lVert\Psi_l^*(I)\rVert_2$ we have $\E[\fro{E_lZ}^2\mid\mathcal F_l]=\mathrm{tr}\,\Psi_l(ZZ^\top)\le\kappa_l\fro{Z}^2$. Hence $\E\fro{\Xh_{l+1}}^2\le2\lVert W\rVert_2^2\bigl[(1-\alpha)\E\fro{\Xh_l}^2+\alpha\fro{X_l}^2\bigr]+2\kappa_l\E\fro{\Xh_l}^2$, so all second moments of the activations and of the errors are finite by induction, for both schemes. The corrected weight itself need not be square integrable, since it contains $\Hh_l^{-1}$; the output involves only the projection $P$.

The two terms $WT$ and $E_lZ$ are now square integrable, so the conditional-mean assumption removes their cross terms, and
\[
C_{l+1}^{\qep}\preceq W\,C_l^{\qep}\,W^\top+\Psi_l(G_l^{\qep}),\qquad
G_{l+1}^{\qep}\preceq W\bigl[(1-\alpha)G_l^{\qep}+\alpha\,X_lX_l^\top\bigr]W^\top+\Psi_l(G_l^{\qep}),
\]
using \eqref{eq:psd1} and \eqref{eq:psd2}. For BASE, $\Xh_{l+1}=(W+E_l)\Xh_l$ gives the exact recursions $C_{l+1}^{\base}=WC_l^{\base}W^\top\allowbreak+\Psi_l(G_l^{\base})$ and $G_{l+1}^{\base}=WG_l^{\base}W^\top\allowbreak+\Psi_l(G_l^{\base})$. BASE is unbiased, $\E\Xh_l^{\base}=X_l$, hence $G_l^{\base}=X_lX_l^\top+C_l^{\base}\allowbreak\succeq X_lX_l^\top$.

Both orderings hold with equality at $l=1$. If they hold at $l$, then, $\Psi_l$ being positive, $C_{l+1}^{\qep}\preceq WC_l^{\base}W^\top+\Psi_l(G_l^{\base})=C_{l+1}^{\base}$, and, using $X_lX_l^\top\preceq G_l^{\base}$, $G_{l+1}^{\qep}\preceq WG_l^{\base}W^\top+\Psi_l(G_l^{\base})=G_{l+1}^{\base}$. Taking traces gives the last claim.
\end{proof}

The induction carries the pair $(C_l,G_l)$; the uncentred moment is needed because the new perturbation enters multiplied by the input. No inclusion between the row spaces of $X_l$ and $\Xh_l$ is assumed: the term $X(I-P)X^\top$ in \eqref{eq:psd2} is exactly what the missing inclusion costs, and it has the right sign. At a layer of positive strength, invertibility of $\Hh_l$ requires $\Xh_l$ to have full row rank; layers of zero strength are not restricted. A linear map cannot increase rank, so a layer whose output feeds a later corrected layer may be narrower than its predecessors but not wider than the rank of its input; the last layer is not restricted. If $\Hh_l^{-1}$ is replaced by the Moore--Penrose pseudo-inverse, $P_l$ is still the orthogonal projection onto the row space of $\Xh_l$, the proof is unchanged and no restriction on widths remains; that is a different algorithm from the one in \cite{qep}.

The two inequalities in the proof discard non-negative terms. Keeping them gives the gain exactly.

\begin{proposition}[Exact gain]\label{prop:gain}
Assume the hypotheses of Theorem~\ref{thm:linear}. Along the QEP trajectory put, for layers with $\alpha_l>0$,
\begin{align*}
S_l&=\alpha_l(2-\alpha_l)\,W_l\,\E\bigl[\delta_lP_l\delta_l^\top\bigr]W_l^\top,\\
B_l&=W_l\,\E\bigl[\alpha_l\,X_l(I-P_l)X_l^\top+\alpha_l(1-\alpha_l)(\Xh_l-X_lP_l)(\Xh_l-X_lP_l)^\top\bigr]W_l^\top,
\end{align*}
and $S_l=B_l=0$ when $\alpha_l=0$. Then $\Delta_l=C_l^{\base}-C_l^{\qep}$ and $\Lambda_l=G_l^{\base}-G_l^{\qep}$ satisfy $\Delta_1=\Lambda_1=0$ and
\begin{align}
\Delta_{l+1}&=W_l\Delta_lW_l^\top+\Psi_l(\Lambda_l)+S_l,\label{eq:gainD}\\
\Lambda_{l+1}&=(1-\alpha_l)\,W_l\Lambda_lW_l^\top+\Psi_l(\Lambda_l)+\alpha_l\,W_lC_l^{\base}W_l^\top+B_l,\label{eq:gainL}
\end{align}
in which every term is positive semidefinite. In particular
\begin{equation}\label{eq:gainlb}
\E\fro{\delta_{L+1}^{\base}}^2-\E\fro{\delta_{L+1}^{\qep}}^2\ \ge\ \sum_{l:\,\alpha_l>0}\alpha_l(2-\alpha_l)\,\E\fro{W_L\cdots W_l\,\delta_l^{\qep}P_l}^2,
\end{equation}
where the sum runs over the layers of positive strength only, the projection $P_l$ not being defined elsewhere; the comparison of Theorem~\ref{thm:linear} is strict as soon as one term on the right is non-zero.
\end{proposition}

\begin{proof}
The relations \eqref{eq:psd1} and \eqref{eq:psd2} are identities. Keeping their right-hand sides in the proof of Theorem~\ref{thm:linear} turns the two inequalities for QEP into $C_{l+1}^{\qep}=W_lC_l^{\qep}W_l^\top+\Psi_l(G_l^{\qep})-S_l$ and $G_{l+1}^{\qep}=W_l[(1-\alpha_l)G_l^{\qep}+\alpha_lX_lX_l^\top]W_l^\top+\Psi_l(G_l^{\qep})-B_l$; the expectations are finite because $P_l$ is a projection, and $P_l$ is not separated from $\delta_l$. Subtracting from the exact recursions for BASE gives \eqref{eq:gainD} at once, and \eqref{eq:gainL} after using $G_l^{\base}=X_lX_l^\top+C_l^{\base}$, which yields $G_l^{\base}-(1-\alpha_l)G_l^{\qep}-\alpha_lX_lX_l^\top=(1-\alpha_l)\Lambda_l+\alpha_lC_l^{\base}$. Positivity of $\Delta_l$ and $\Lambda_l$ is Theorem~\ref{thm:linear}. Dropping $\Psi_l(\Lambda_l)\succeq0$ in \eqref{eq:gainD} and iterating gives $\Delta_{L+1}\succeq\sum_lW_L\cdots W_{l+1}S_lW_{l+1}^\top\cdots W_L^\top$, whose trace is the right-hand side of \eqref{eq:gainlb}, the terms with $\alpha_l=0$ being zero.
\end{proof}

The two recursions separate two effects. $S_l$ is the error removed at layer $l$. The source in \eqref{eq:gainL} is a reduction of the energy that excites later rounding errors; it reaches the output only through the maps $\Psi$ of later layers, and is not seen by a per-layer comparison. Since the recursions are linear with positive coefficients, the gain is a sum of non-negative contributions of the sources $S_l$ and $\alpha_lW_lC_l^{\base}W_l^\top+B_l$, each weighted by a positive semidefinite matrix that depends only on the teacher, the maps $\Psi_l$ and the strengths; the bound \eqref{eq:gainlb} keeps the first kind only. A gain at an intermediate layer need not reach the output: later weights and noise maps may ignore the directions in which it occurs.

\begin{corollary}[Schedules with a common row space]\label{cor:sched}
Assume the hypotheses of Theorem~\ref{thm:linear} and let $a$ and $b$ be two schedules fixed in advance with $0\le a_l\le b_l\le1$. Suppose that along both trajectories $X_lP_l=X_l$ almost surely at every layer of positive strength; this is the case when $X_1$ has full row rank and all teacher and quantized weight matrices are square and invertible. Then $C_l^{b}\preceq C_l^{a}$ and $G_l^{b}\preceq G_l^{a}$ for every $l$. For each such schedule,
\[
C_{l+1}=(1-\alpha_l)^2\,W_lC_lW_l^\top+\Psi_l(C_l)+\Psi_l(X_lX_l^\top).
\]
In any class of fixed schedules that satisfy these hypotheses with the same maps $\Psi_l$ and that contains the full-strength schedule, full strength minimizes the mean-square output error.
\end{corollary}

\begin{proof}
With $X_lP_l=X_l$ the target is $T=(1-\alpha_l)\Xh_l+\alpha_lX_l$, so $\delta_{l+1}=(1-\alpha_l)W_l\delta_l-E_l\Xh_l$. By induction every schedule is unbiased, $\E\delta_l=0$, hence $G_l=X_lX_l^\top+C_l$ and the displayed recursion holds. If $C_l^{b}\preceq C_l^{a}$, positivity of the maps and $(1-b_l)^2\le(1-a_l)^2$ give $C_{l+1}^{b}\preceq C_{l+1}^{a}$.
\end{proof}

\begin{example}[A closed form for the gain]\label{ex:iso}
Fix a calibration input $X_1$ with $d$ rows and full row rank, and a schedule. Let all layers be square of size $d$ and let the teacher weights be independent random matrices, almost surely invertible, with $\E[W_l^\top W_l]=\gamma_lI$ for deterministic constants $0<\gamma_l<\infty$; for instance $W_l=I+V_l$ with the entries of $V_l$ independent and uniform on $[-a_l,a_l]$, $a_l<1/(2d)$, where $\gamma_l=1+da_l^2/3$. Suppose that, conditionally on all teacher weights and on the past, $E_l$ has zero mean and $\E[E_lME_l^\top]=v_l\,\mathrm{tr}(M)\,I$ with deterministic constants $0\le v_l<\infty$ common to the two schemes, and that the quantized weights are almost surely invertible. Put $q_l=\E\fro{\delta_{l+1}}^2/\E\fro{X_{l+1}}^2$, with expectations over the teacher as well, and $\nu_l=dv_l/\gamma_l$. Taking traces and the expectation over the teacher in Corollary~\ref{cor:sched} gives $q_l=[(1-\alpha_l)^2+\nu_l]\,q_{l-1}+\nu_l$ with $q_0=0$, hence
\[
q_L^{\base}=\prod_{l=1}^{L}(1+\nu_l)-1,\qquad q_L^{\qep}=\sum_{j=1}^{L}\prod_{i=j}^{L}\nu_i\quad\text{at full strength}.
\]
For $\nu_l=\nu>0$ the ratio $q_L^{\base}/q_L^{\qep}=\sum_{k<L}(1+\nu)^k\big/\sum_{k<L}\nu^k$ increases with the depth and tends to $L$ as $\nu\downarrow0$ (for $\nu=0$ both errors vanish): at fine resolution full compensation leaves the rounding noise of the last layer only, a gain of $10\log_{10}L$~dB. This is an average over random teachers and a ratio of expectations, in a class where all row spaces coincide; it is not a statement about a fixed network or about bottlenecks.
\end{example}

Three limits should be stated. First, Theorem~\ref{thm:linear} compares any fixed schedule with the uncorrected scheme. Outside the setting of Corollary~\ref{cor:sched}, monotonicity between two non-zero schedules can fail; the condition of that corollary is empty for the zero schedule and so does not exclude this. Take $X=I_4$, $W_1=[\,I_2\mid0\,]$, $W_2=(1,0)$, $W_3=1$, with errors $E_1=[\,0\mid\mathrm{diag}(u_1,u_2)\,]$, $E_2=(0,\eta)$, $E_3=0$, where $u_1,u_2=\pm1$ and $\eta=\pm\tfrac12$ are independent fair signs, and $\alpha_3=1$. The mean-square output error is $\tfrac32$ without correction, $\tfrac{17}{26}$ at $\alpha_2=\tfrac12$ and $\tfrac34$ at $\alpha_2=1$. Second, the theorem does not cover feedback coefficients recomputed along each trajectory, which change $\Psi_l$; strengths chosen after seeing the rounding, which break the fixed schedule and the conditional mean; or overload, which introduces a conditional bias. Third, it does not extend to the bounds $U^M$. For three scalar linear layers with unit weights and input, strengths $(0,1,0)$, and independent errors $\pm t$, $\pm4$, $\pm\tfrac43$ with fair signs and $0<t<\tfrac{6}{67}$, one finds $\E(U^{\qep})^2-\E(U^{\base})^2=t(6-67t)/9>0$. These errors satisfy the hypotheses of the theorem but are not a uniform dither; for uniform dither the comparison of the bounds is not settled here.

\section{ReLU networks: two obstructions and a conditional guarantee}\label{sec:relu}

The comparison of Theorem~\ref{thm:linear} does not survive a nonlinearity. We give the two ways in which it fails and a condition under which it is recovered up to an explicit excess.

\begin{example}[The mean-square comparison fails for ReLU]\label{ex:relu}
Let $L=2$, all quantities scalar, both activations ReLU, $X=1$, $W_1=1$, $W_2=-\tfrac14$, and let both layers use subtractive dither on $\mathbb Z$, with errors $u,v$ independent and uniform on $[-\tfrac12,\tfrac12)$. The teacher output is $0$. With $z=1+u\in[\tfrac12,\tfrac32)$, the uncorrected output is $z\,(v-\tfrac14)_+$; at $\alpha_2=1$ the QEP target is $-1/(4z)$ and the QEP output is $(vz-\tfrac14)_+$. Then
\[
\E\bigl[(\Xh_3^{\base})^2\bigr]=\frac{13}{2304}\approx0.00564,\qquad
\E\bigl[(\Xh_3^{\qep})^2\bigr]=\frac1{72}-\frac{\log3}{192}\approx0.00817 .
\]
No target is out of range and all Gram matrices are positive. Given the first layer, the correction restores the conditional mean of the pre-activation to its teacher value $-\tfrac14$ (without it the conditional mean is $-z/4$). The conditional variance, $z^2/12$, is the same with and without the correction; the unconditional variance of the pre-activation drops from $\tfrac{55}{576}$ to $\tfrac{13}{144}$. The teacher's ReLU is inactive there, so what reaches the output is the upper tail of the noise. The corrected output is positive less often, with probability $\tfrac12-\tfrac{\log3}{4}<\tfrac14$, but its positive tail carries more energy. A zero conditional mean before the activation does not survive the activation.
\end{example}

Example~\ref{ex:relu} has noise crossing the activation boundary. That is not the only obstruction. A unit with a fixed 0/1 activation pattern across the calibration samples reads its pre-activation error through a diagonal mask $D$, while the correction mixes samples through $P$. The next proposition says when the correction is compatible with the mask for every possible inherited error.

\begin{proposition}[Masked contraction]\label{prop:commute}
Let $P$ be an orthogonal projection, $D$ a diagonal matrix with entries in $\{0,1\}$, and $0<\alpha\le1$. Then $\lVert a(I-\alpha P)D\rVert_2\le\lVert aD\rVert_2$ for every row vector $a$ if and only if $PD=DP$.
\end{proposition}

\begin{proof}
If $PD=DP$ then $a(I-\alpha P)D=aD(I-\alpha P)$ and $I-\alpha P$ is a contraction. Conversely, for every $a$ with $aD=0$ the inequality forces $aPD=0$, that is, $(I-D)PD=0$; transposing, $DP(I-D)=0$, hence $PD=DPD=DP$.
\end{proof}

The proof uses only that $D$ is a symmetric idempotent. The condition characterizes contraction for all inherited errors; in a given network the inherited error is constrained by the teacher and the input, so commutation is not necessary for the mean-square comparison in that network. The same computation gives a weighted form: for any symmetric positive semidefinite $M$, the inequality $r(I-\alpha P)M(I-\alpha P)r^\top\le rMr^\top$ holds for every row vector $r$ if and only if $PM=MP$; in an orthonormal basis adapted to $P$, with $M=\bigl(\begin{smallmatrix}M_{11}&M_{12}\\M_{12}^\top&M_{22}\end{smallmatrix}\bigr)$, the difference $M-(I-\alpha P)M(I-\alpha P)$ has blocks $\alpha(2-\alpha)M_{11}$, $\alpha M_{12}$ and $0$, and is positive semidefinite exactly when $M_{12}=0$. The activation mask is one instance of a metric on the sample space.

The defect of commutation can be computed from the calibration data: $\fro{PD-DP}^2=2\sum_{i:\,d_i=1}\sum_{j:\,d_j=0}P_{ij}^2$, the squared coupling, through the fit, between samples on which the unit is active and samples on which it is inactive. Here $P$ acts on the index space of the calibration samples; the coupling comes from fitting one weight matrix to all samples, not from any mixing of inputs at inference. Theorem~\ref{thm:relu} below uses the spectral norm of the commutator, which this Frobenius norm bounds from above.

\begin{example}[Failure without any change of activation]\label{ex:noflip}
Let $L=2$ with both activations ReLU,
\[
X=\begin{pmatrix}1&-1&1\\-1&1&1\end{pmatrix},\qquad W_1=I_2,\qquad W_2=\begin{pmatrix}1&-2\\-1&-1\end{pmatrix},
\]
and let every weight of both layers be quantized with subtractive dither of the same step, with errors independent and uniform on $[-t,t]$; $\alpha_2=1$. The teacher pre-activations of the second layer are bounded away from zero by $1$. For $0<t\le\tfrac1{10}$, in both schemes and for every realization of the errors, every unit has the teacher's activation state on every sample, and all Gram matrices and quantized weights are invertible. Nevertheless
\[
\E\fro{\delta_3^{\qep}}^2-\E\fro{\delta_3^{\base}}^2\ \ge\ t^2\Bigl(\tfrac{2}{27}-42\,t(1+t)\Bigr)>0\qquad\text{for }0<t\le\tfrac1{600},
\]
and the difference equals $\tfrac{2}{27}t^2+O(t^3)$; the proof is in Appendix~\ref{app:noflip}.
Only the first output unit on the first sample is active. On that entry, with $h_{ij}$ the normalized first-layer errors, the error inherited from the first layer is $t\,(h_{11}-h_{12})$ to first order without correction and $t\,(-2h_{12}+4h_{21})/3$ with it; the second-layer error contributes the same term to both. The correction has moved error from samples where the unit is closed onto the sample where it is open. Here $P$ does not commute with the activation masks.
\end{example}

When the projections respect the activation patterns the linear argument goes through, and when they nearly do the comparison holds up to an explicit cost.

Partition the calibration samples into cells, with $J_c$ the diagonal 0/1 projection onto cell $c$, such that within each cell all samples have the same teacher activation pattern at every layer. Let $D_{l,c}$ be the diagonal 0/1 pattern of cell $c$ at the output of layer $l$, and put $F_{l,c}=D_{l,c}W_l$ and $\Psi_{l,c}(T)=D_{l,c}\Psi_l(T)D_{l,c}$. Along the QEP trajectory let $k_{l,c}=\lVert P_lJ_c-J_cP_l\rVert_2$ if $l\ge2$ and $\alpha_l>0$, and $k_{l,c}=0$ otherwise. Define deterministic positive semidefinite matrices by $K_{1,c}=V_{1,c}=0$ and
\begin{align}
K_{l+1,c}&=F_{l,c}K_{l,c}F_{l,c}^\top+\Psi_{l,c}(V_{l,c})+\alpha_l\,F_{l,c}\,\E\bigl[k_{l,c}\,\delta_l^{\qep}(\delta_l^{\qep})^\top\bigr]F_{l,c}^\top,\label{eq:defK}\\
V_{l+1,c}&=(1-\alpha_l)F_{l,c}V_{l,c}F_{l,c}^\top+\Psi_{l,c}(V_{l,c})+\alpha_l\,\E[k_{l,c}]\,F_{l,c}X_lX_l^\top F_{l,c}^\top.\label{eq:defV}
\end{align}

\begin{theorem}[ReLU networks with stable activations]\label{thm:relu}
Let all activations be ReLU and let the strengths be fixed in advance. Assume:
\begin{enumerate}[label=(\alph*),leftmargin=*,itemsep=1pt]
\item almost surely both quantized trajectories have the teacher's activation pattern at every layer and on every sample;
\item the hypotheses of Theorem~\ref{thm:linear} on $E_l$ and $\Hh_l$.
\end{enumerate}
Then, for every $l$ and $c$,
\begin{align*}
\E\bigl[\delta_l^{\qep}J_c(\delta_l^{\qep})^\top\bigr]&\preceq\E\bigl[\delta_l^{\base}J_c(\delta_l^{\base})^\top\bigr]+K_{l,c},\\
\E\bigl[\Xh_l^{\qep}J_c(\Xh_l^{\qep})^\top\bigr]&\preceq\E\bigl[\Xh_l^{\base}J_c(\Xh_l^{\base})^\top\bigr]+V_{l,c},
\end{align*}
and therefore
\[
\E\fro{\delta_{L+1}^{\qep}}^2\le\E\fro{\delta_{L+1}^{\base}}^2+\sum_c\mathrm{tr}\,K_{L+1,c}.
\]
In particular, if $P_lJ_c=J_cP_l$ almost surely for every $c$ and every $l\ge2$ with $\alpha_l>0$, then $\E\fro{\delta_{L+1}^{\qep}}^2\le\E\fro{\delta_{L+1}^{\base}}^2$.
\end{theorem}

\begin{proof}
Let $P$ and $J$ be orthogonal projections, $0\le\alpha\le1$ and $k=\lVert PJ-JP\rVert_2$. Then
\begin{equation}\label{eq:comm}
(I-\alpha P)J(I-\alpha P)\preceq J+\alpha k\,I,\qquad PJP\preceq J+k\,I .
\end{equation}
Indeed, write $x=p+q$ with $p=Px$. Since $PJ(I-P)=P(PJ-JP)(I-P)$, we have $\lvert p^\top Jq\rvert\le k\lVert p\rVert\lVert q\rVert\le\tfrac12k\lVert x\rVert^2$, and the two quadratic forms to be bounded equal $-\alpha(2-\alpha)\,p^\top Jp-2\alpha\,p^\top Jq$ and $-q^\top Jq-2\,p^\top Jq$.

Fix a layer and a cell, with the notation of the proof of Theorem~\ref{thm:linear} and $J=J_c$, $k=k_{l,c}$. From $T-X=(Z-X)(I-\alpha P)$ and \eqref{eq:comm},
\[
(T-X)J(T-X)^\top\preceq(Z-X)J(Z-X)^\top+\alpha k\,(Z-X)(Z-X)^\top .
\]
Since $(1-\alpha)ZJZ^\top+\alpha\,XPJPX^\top-TJT^\top=\alpha(1-\alpha)(Z-XP)J(Z-XP)^\top\succeq0$, \eqref{eq:comm} also gives
\[
TJT^\top\preceq(1-\alpha)ZJZ^\top+\alpha\,XJX^\top+\alpha k\,XX^\top .
\]
At the first layer and at layers with $\alpha=0$ we have $T=Z$ and both hold with $k=0$. No Gram matrix of $ZJ_c$ is inverted.

By (a), on the samples of cell $c$ the next activations are $F_{l,c}T+D_{l,c}E_lZ$ for QEP, $F_{l,c}Z+D_{l,c}E_lZ$ for BASE and $F_{l,c}X$ for the teacher. Write $C_{l,c}^M$ and $G_{l,c}^M$ for the two cell moments in the statement. Integrability is obtained as in Theorem~\ref{thm:linear}, since ReLU does not increase norms. The conditional-mean assumption removes the cross terms, so
\begin{align*}
C_{l+1,c}^{\qep}&\preceq F_{l,c}C_{l,c}^{\qep}F_{l,c}^\top+\Psi_{l,c}(G_{l,c}^{\qep})+\alpha_lF_{l,c}\E\bigl[k_{l,c}\delta_l^{\qep}(\delta_l^{\qep})^\top\bigr]F_{l,c}^\top,\\
G_{l+1,c}^{\qep}&\preceq F_{l,c}\bigl[(1-\alpha_l)G_{l,c}^{\qep}+\alpha_lX_lJ_cX_l^\top\bigr]F_{l,c}^\top+\Psi_{l,c}(G_{l,c}^{\qep})+\alpha_l\E[k_{l,c}]F_{l,c}X_lX_l^\top F_{l,c}^\top .
\end{align*}
The commutator norm $k_{l,c}$ is random and correlated with $\delta_l^{\qep}$, so the expectation of their product is not factored. For BASE the recursions are exact, $C_{l+1,c}^{\base}=F_{l,c}C_{l,c}^{\base}F_{l,c}^\top+\Psi_{l,c}(G_{l,c}^{\base})$ and $G_{l+1,c}^{\base}=F_{l,c}G_{l,c}^{\base}F_{l,c}^\top+\Psi_{l,c}(G_{l,c}^{\base})$, and by (a) BASE is unbiased on the cell, so $G_{l,c}^{\base}\succeq X_lJ_cX_l^\top$. Both claims hold with equality at $l=1$. If they hold at $l$, inserting them in the two inequalities and using the positivity of $\Psi_{l,c}$ gives them at $l+1$, with \eqref{eq:defK} and \eqref{eq:defV}. The squared Frobenius error is the sum over cells of $\mathrm{tr}\,C_{L+1,c}$.
\end{proof}

Condition (a) must hold almost surely, not on an event selected afterwards, since conditioning on the absence of activation changes does not in general preserve the zero conditional mean; it can be certified from the margins of the teacher pre-activations and a bound on the errors. Commutation says that the correction does not project samples of one activation pattern onto another. Stability and commutation are independent: with a single sample commutation is automatic, and Example~\ref{ex:relu} still fails because (a) does; Example~\ref{ex:noflip} satisfies (a) and has non-commuting projections.

Exact commutation is restrictive. Together with the invertibility of $\Hh_l$ it is equivalent to $\sum_c\mathrm{rank}(\Xh_lJ_c)=d_l$, where $d_l$ is the number of rows of $\Xh_l$: the activations of the different cells must span complementary subspaces, so the number of cells with non-zero activations at that layer is at most $d_l$. It holds when all samples share one pattern, and it fails as soon as two cells have non-zero activations and one of them spans the whole space. It is not vacuous. With
\[
X=\begin{pmatrix}1&2&-1&-1\\-1&-1&1&2\end{pmatrix},\quad W_1=I_2,\quad W_2=\begin{pmatrix}1&1\\1&-1\end{pmatrix},\quad W_3=\begin{pmatrix}1&-3\\1&1\end{pmatrix},
\]
and cells $\{1,2\}$ and $\{3,4\}$, the two cells have different activation patterns at every layer, and all hypotheses hold with commuting projections for independent errors $\pm\tfrac1{20}$ on every weight. Exact enumeration of the $4096$ realizations gives mean-square output error $3287607/16000000\approx0.2055$ without correction and $\approx0.0532$ with $\alpha=(0,1,1)$.

The defect term makes the commuting case stable under perturbation. If $Z_0$ has full row rank with smallest singular value at least $s$, $P_0$ is its row-space projection and $\lVert Z-Z_0\rVert_2\le\eta<s$, then the row-space projection $P$ of $Z$ satisfies $\lVert P-P_0\rVert_2\le2\eta/(s-\eta)$, and hence $\lVert PJ-JP\rVert_2\le\lVert P_0J-JP_0\rVert_2+2\lVert P-P_0\rVert_2$ for every projection $J$. A lower bound on the singular values is needed: the row-space projection of $\bigl(\begin{smallmatrix}1&0&0\\0&\eta&\eta\end{smallmatrix}\bigr)$ has commutator norm $\tfrac12$ with $\mathrm{diag}(1,1,0)$ for every $\eta>0$.

\begin{corollary}[Small errors]\label{cor:hr}
Fix a ReLU network, the calibration input, the cells and the strengths. Suppose that all teacher pre-activations are non-zero, that $X_l$ has full row rank whenever $\alpha_l>0$, and that the teacher projections $P_l^0$ satisfy $\lVert P_l^0J_c-J_cP_l^0\rVert_2\le\varepsilon_0$ for $l\ge2$ with $\alpha_l>0$. Let $E_l=tN_l$ with $\lVert N_l\rVert_2$ bounded, $\E[N_l\mid\mathcal F_l]=0$ and $\E[N_lTN_l^\top\mid\mathcal F_l]=\Psi_l^0(T)$ for a deterministic positive linear map common to both schemes. Then there are $t_0>0$ and a constant $\kappa$, depending on the network, the input and the bound on $N_l$, such that for $0<t<t_0$
\[
\E\fro{\delta_{L+1}^{\qep}}^2\le\E\fro{\delta_{L+1}^{\base}}^2+\kappa\,(\varepsilon_0t^2+t^3).
\]
\end{corollary}

\begin{proof}
For both schemes $\fro{\delta_{l+1}}\le(\lVert W_l\rVert_2+t\lVert N_l\rVert_2)\fro{\delta_l}+t\lVert N_l\rVert_2\fro{X_l}$ on every realization, since the entrywise ReLU is $1$-Lipschitz in the Frobenius norm and $I-\alpha P$ is a contraction; hence $\fro{\delta_l}=O(t)$ uniformly, and $\lVert\delta_l\rVert_2\le\fro{\delta_l}$. For small $t$ the margins of the teacher give (a), the quantized activations entering layers of positive strength have full row rank with singular values bounded below, and by the perturbation bound above $k_{l,c}\le\varepsilon_0+O(t)$. In \eqref{eq:defV} the source is $O(\varepsilon_0+t)$, so $V_{l,c}=O(\varepsilon_0+t)$. In \eqref{eq:defK} the source is $O((\varepsilon_0+t)t^2)$ and $\Psi_{l,c}(V_{l,c})=t^2\,D_{l,c}\Psi_l^0(V_{l,c})D_{l,c}=O((\varepsilon_0+t)t^2)$. The depth is fixed, so $\mathrm{tr}\,K_{L+1,c}=O(\varepsilon_0t^2+t^3)$.
\end{proof}

When the teacher projections commute exactly, the excess is $O(t^3)$, an upper bound, while the errors themselves are in general of order $t^2$, although the projections along the quantized trajectory do not commute in general. An instance: $X_1$ with rows $(2,2,1,1)$, $(1,1,2,2)$, $(0,0,\eta,-\eta)$, $W_1=[\,I_2\mid0\,]$, $W_2=\bigl(\begin{smallmatrix}1&-1\\-1&3/2\end{smallmatrix}\bigr)$, cells $\{1,2\}$ and $\{3,4\}$, $\alpha=(0,1)$, $0<\eta\le\tfrac1{10}$ and independent errors uniform on $[-t,t]$ with $t\le10^{-3}$: the two cells have different patterns, (a) holds on the whole support, and $0<k_{2,c}\le8\eta t/(1-2\eta t)$ outside a null set.

The bound of Theorem~\ref{thm:relu} can be nearly attained and can be loose. Three saved instances with independent sign-valued additive errors have certificates combining exact rational output risks with rational intervals for the commutator norms. The stored binary parameters are embedded exactly as rationals. Every sign realization is checked for agreement with the teacher's masks along both trajectories and for the required Gram inverses at corrected layers. The certificates use the stated recursion, with zero first-layer commutator contribution. They support a ratio of actual excess to $\sum_c\mathrm{tr}\,K_{L+1,c}$ of approximately $0.999994$, a three-layer case in which dropping $\Psi_{l,c}(V_{l,c})$ makes the bound fail, and another ratio of approximately $0.74$.

\begingroup\emergencystretch=2em
The accompanying code repository's \nolinkurl{numerics/saved_witnesses/} directory contains the configurations. The near-attainment case is \nolinkurl{checkG_climb_3_1.json}; the omitted-term case is \nolinkurl{checkG_climb_variant_a_1.json}; the third is \nolinkurl{checkG_search_seed5.json}. Their rational risks and interval endpoints are indexed by the same filenames in \nolinkurl{SAVED_WITNESS_REPLAY.json}; \nolinkurl{replay_fixed_witnesses.py} replays only these saved inputs. These instances support near attainment and the need for the later-noise term, not optimality of the constant or a guarantee for actual quantizers. In Example~\ref{ex:noflip} the bound exceeds the actual excess by more than an order of magnitude. Neither result removes condition (a), and we do not claim that trained networks satisfy it.\par\endgroup

\section{Fixed rounding errors and the published guarantee}\label{sec:counter}

Sections~\ref{sec:network} and~\ref{sec:relu} establish stochastic network comparisons under explicit assumptions. In contrast, the guarantee published with QEP is asserted for arbitrary fixed perturbation matrices shared by the two schemes. Under this deterministic convention, the comparison fails already for two scalar linear layers. Here the corrected scheme denotes the idealized correction in that guarantee; actual codebook minimizers are treated separately in Section~\ref{sec:actual}.

To our knowledge the only published statement of this kind is the guarantee given with Quantization Error Propagation (QEP) \cite{qep}, which adds a closed-form correction to the weights before quantizing them, scaled by a propagation strength $\alpha_l\in[0,1]$. Theorem~5.2 of \cite{qep} states that the output error of the propagated scheme is bounded by that of the uncorrected scheme; its formal version, Theorem~B.5, states that a Lipschitz upper bound $U^{\qep}$ on the output error satisfies $U^{\qep}\le U^{\base}$ for any fixed perturbation matrices and every choice of strengths. Proposition~5.4 adds that the output error decreases monotonically as each $\alpha_l$ approaches one. In the published discussion of the paper the authors describe Theorem~5.2 as ``a rigorous guarantee that solving the layer-wise local optimization problem \ldots{} reduces the final quantization error''\footnote{Authors' rebuttal of 30 July 2025 in the public discussion of \cite{qep}, OpenReview forum \texttt{a3l3K9khbL}, under ``Weakness and Question 2''; archived copy consulted on 4 October 2026.} and use it to distinguish QEP from a concurrent method with the same local objective. 

\paragraph{Claims under examination.}
\begin{itemize}[leftmargin=*,itemsep=2pt]
\item \cite[Theorem~B.5]{qep}: for every choice of $\{\alpha_l\}\subset[0,1]$, $U^{\qep}\le U^{\base}$, via the per-layer inequality $\fro{R_l^{\qep}}\le\fro{R_l^{\base}}$ for all $l$.
\item \cite[Theorem~5.2, informal]{qep}: $\fro{f_\theta(X)-f_{\widehat\theta_{\qep}}(X)}\le\fro{f_\theta(X)-f_{\widehat\theta_{\base}}(X)}$ for the quantizers defined by the two layer objectives.
\item \cite[Proposition~5.4]{qep}: the output error of QEP decreases monotonically as each $\alpha_l$ approaches $1$.
\end{itemize}
The appendix of \cite{qep} works to first order in $\{E_l\}$ and writes the same symbol $E_l$ for the perturbation in both schemes. The counterexample of Theorem~\ref{thm:counter} respects both conventions.

\begin{theorem}[Counterexample to Theorem~B.5 of \cite{qep}]\label{thm:counter}
Let $L=2$, all matrices $1\times1$, $\sigma_1=\sigma_2=\mathrm{id}$, $\gamma_1=\gamma_2=1$, $W_1=W_2=1$, $X=1$, $\alpha_1=\alpha_2=1$, and $E=(\varepsilon,-\varepsilon)$ with $0<\varepsilon<\tfrac12$. Then all hypotheses of Theorem~B.5 hold, including invertibility of the Gram matrices and of all quantized weights, and
\begin{align*}
R_1^{\base}&=R_1^{\qep}=-\varepsilon, & \Xh_2^{\base}&=\Xh_2^{\qep}=1+\varepsilon,\\
R_2^{\base}&=\varepsilon^2, & R_2^{\qep}&=\varepsilon+\varepsilon^2,\\
U^{\base}&=\varepsilon+\varepsilon^2, & U^{\qep}&=2\varepsilon+\varepsilon^2,\\
\lvert X_3-\Xh_3^{\base}\rvert&=\varepsilon^2, & \lvert X_3-\Xh_3^{\qep}\rvert&=\varepsilon+\varepsilon^2.
\end{align*}
Hence $U^{\qep}-U^{\base}=\varepsilon>0$ and the QEP output error is strictly larger, for every $0<\varepsilon<\tfrac12$.
\end{theorem}

\begin{proof}
In the first layer $\delta_1=0$, so the correction vanishes and both schemes use $\Wh_1=1+\varepsilon$, giving $\Xh_2=1+\varepsilon$ and $R_1=1-(1+\varepsilon)=-\varepsilon$. In the second layer $\delta_2=-\varepsilon$ and $\Hh_2=(1+\varepsilon)^2>0$.

\emph{BASE.} $\Wh_2=1-\varepsilon$, output $(1-\varepsilon)(1+\varepsilon)=1-\varepsilon^2$, so $R_2=\varepsilon^2$.

\emph{QEP.} $W_2^*=1+\delta_2\Xh_2/\Xh_2^2=1/(1+\varepsilon)$, $\Wh_2=1/(1+\varepsilon)-\varepsilon$, output $1-\varepsilon(1+\varepsilon)=1-\varepsilon-\varepsilon^2$, so $R_2=\varepsilon+\varepsilon^2$.

With $\gamma_s=\lVert W_s\rVert_2=1$, $U^M=\lvert R_1^M\rvert+\lvert R_2^M\rvert$. Since $\sigma_2$ is the identity, the output error equals $\lvert R_2^M\rvert$. For $0<\varepsilon<\tfrac12$ the quantized weights $1+\varepsilon$, $1-\varepsilon$ and $1/(1+\varepsilon)-\varepsilon$ are positive. (The algebraic identities hold for every $\varepsilon>0$; the QEP second-layer weight vanishes at $\varepsilon=(\sqrt5-1)/2$ and the BASE one at $\varepsilon=1$.)
\end{proof}

The gap is first order: $\max_k\lvert E_k\rvert=\varepsilon$ and $(U^{\qep}-U^{\base})/\varepsilon=1$ throughout, while the two bounds and the two output errors differ from their first-order expressions by exactly $\varepsilon^2$. The first-order convention of \cite{qep} therefore cannot absorb it. Both schemes see the same second-layer input, so the failure does not come from comparing different trajectories.

\begin{remark}[The sign matters]
With $E=(\varepsilon,\varepsilon)$ the same computation gives $U^{\base}=3\varepsilon+\varepsilon^2$, $U^{\qep}=2\varepsilon+\varepsilon^2$ and output errors $2\varepsilon+\varepsilon^2$ against $\varepsilon+\varepsilon^2$: QEP is better. The ordering depends on the relation between the inherited error and the new perturbation, which is what the identity of Section~\ref{sec:identity} makes precise.
\end{remark}

\begin{theorem}[Failure of monotonicity in the shared-perturbation model]\label{thm:alpha}
In the setting of Theorem~\ref{thm:counter}, keep $\alpha_1=1$ and let $\alpha_2=\alpha\in[0,1]$. Then the QEP output is $1-\alpha\varepsilon-\varepsilon^2$ and the output error is $\alpha\varepsilon+\varepsilon^2$. For $0\le\alpha<\beta\le1$ the error at $\beta$ exceeds the error at $\alpha$ by $(\beta-\alpha)\varepsilon>0$.
\end{theorem}

\begin{proof}
$W_2^*(\alpha)=1-\alpha\varepsilon/(1+\varepsilon)$, so $\Wh_2=1-\alpha\varepsilon/(1+\varepsilon)-\varepsilon$ and the output is $(1+\varepsilon)\Wh_2=1-\alpha\varepsilon-\varepsilon^2$.
\end{proof}

This contradicts Proposition~5.4 of \cite{qep} read in the shared-perturbation model of its appendix. It does not contradict Corollary~B.6 of \cite{qep}: that corollary bounds the residual by $\fro{W_l\delta_l(I-\alpha P_l)}+\lVert E_l\rVert_2\fro{\Xh_l}$, which here equals $(2-\alpha)\varepsilon+\varepsilon^2$, is decreasing in $\alpha$, and remains a valid upper bound on $\alpha\varepsilon+\varepsilon^2$ for $\alpha\le1$. The bound is monotone; the error it bounds is not.

\paragraph{Where the proof breaks.}
Let $P_l=\Xh_l^\top\Hh_l^{-1}\Xh_l$, the orthogonal projection onto the row space of $\Xh_l$. The proof of Theorem~B.5 correctly derives
\[
R_l^{\qep}=W_l\delta_l(I-\alpha_lP_l)-E_l\Xh_l,\qquad
\widetilde R_l^{\base}:=W_l\delta_l-E_l\Xh_l,
\]
where $\widetilde R_l^{\base}$ is the BASE residual on the same $\Xh_l$, and correctly notes $\fro{W_l\delta_l(I-\alpha_lP_l)}\le\fro{W_l\delta_l}$. It then asserts, ``by the triangle inequality'',
\[
\fro{W_l\delta_l(I-\alpha_lP_l)-E_l\Xh_l}\le\fro{\widetilde R_l^{\base}}.
\]
This does not follow. Shrinking one term of a difference does not shrink the difference. In Theorem~\ref{thm:counter}, $W_2\delta_2=-\varepsilon$ and $E_2\Xh_2=-\varepsilon-\varepsilon^2$ nearly cancel in the BASE residual; QEP removes the first term and leaves the second intact.

The proof contains a second unjustified step: having fixed $\Xh_l=\Xh_l^{\qep}$, it concludes by ``taking $\Xh_l=\Xh_l^{\base}$'', although the two trajectories differ in general. Theorem~\ref{thm:counter} isolates the first gap, since there the trajectories coincide up to the layer in question. The second gap causes failure on its own, as Example~\ref{ex:traj} shows.

Three further facts about the shared-perturbation statements are established in Appendix~\ref{app:more}. The failure is not an artifact of $P=I$ or of scalar networks: it occurs with a proper projection, lifts to matrices, and occurs in a family that does not reduce to independent scalar channels (Example~\ref{ex:proj}). The substitution of one trajectory for the other fails on its own, in a three-layer network where the first step causes no harm (Example~\ref{ex:traj}). The composite inequality $\fro{\delta_{L+1}^{\qep}}\le U^{\base}$, also asserted in Theorem~B.5 of \cite{qep}, fails for finite perturbations (Example~\ref{ex:composite}); it holds once each factor $\lVert W_s\rVert_2$ in $U^{\base}$ is replaced by $\lVert W_s\rVert_2+\lVert E_s\rVert_2$, and therefore to first order.

\section{Actual nearest-codeword quantizers}\label{sec:actual}

The shared-perturbation model fixes $E_l$ independently of the target. Real quantizers do not behave this way.

\begin{remark}[The correction is piecewise constant]\label{rem:piecewise}
Let $Q$ be a deterministic nearest-codeword map with a fixed codebook and scale, and $\Wh(\alpha)=Q(W+\alpha K)$ with $K=W\delta\Xh^\top\Hh^{-1}$. With the layer input and all other layers held fixed, on any interval of $\alpha$ on which no entry of $W+\alpha K$ crosses a decision boundary, $\Wh(\alpha)$ is constant, so the quantized layer and its loss are constant. Within such an interval the change in the continuous target is cancelled exactly by the change in rounding error: for $\alpha,\alpha_0$ in the same interval, $B_\alpha=B_{\alpha_0}-(\alpha-\alpha_0)K\Xh$. The loss therefore changes only at boundary crossings, which are determined by the path of the target and by the codebook; the quadratic of Proposition~\ref{prop:identity}, which holds $B$ fixed, does not describe them.
\end{remark}

\begin{remark}[No local harm at full strength with exact minimization]\label{rem:local}
For any candidate $Q$ and common input $\Xh$,
\[
\fro{WX-Q\Xh}^2=\fro{W\delta(I-P)}^2+\fro{(W^*(1)-Q)\Xh}^2 .
\]
Hence an exact minimizer over the codebook of the QEP objective with $\alpha=1$ attains a layer residual no larger than any other codeword assignment on that input, including the one chosen by the uncorrected objective. This is the standard least-squares decomposition and is implicit in \cite[Eq.~(5)]{qep}. It shows that the shared-perturbation instance of Theorem~\ref{thm:counter} cannot arise from exact minimization of both objectives on a common input with a common codebook. It says nothing about the network output.
\end{remark}

The next result shows that output-level domination fails for exact minimizers.

\begin{theorem}[An actual finite-codebook family]\label{thm:actual}
Let the codebook be $\{0,1,2,3\}$ and consider a three-layer network with scalar weights $a,b,c>0$, ReLU activations and two calibration samples $X=(-1,1)$. Let each layer be quantized by exact minimization over the codebook, the uncorrected scheme using the objective of \cite[Eq.~(1)]{qep} with its own quantized inputs and QEP using \cite[Eq.~(3)]{qep} with $\alpha=1$ and no damping. Suppose
\[
\tfrac12<a<\tfrac32,\qquad \tfrac32<b<\tfrac52,\qquad \tfrac52<ab<3,\qquad \tfrac12<c<\tfrac32,\qquad \tfrac32<abc<\tfrac72 .
\]
Then every layer minimizer in both schemes is unique, the uncorrected scheme selects $(1,2,1)$ and QEP selects $(1,3,1)$, the continuous targets lie strictly inside $(0,3)$, and all input Gram matrices are positive. Writing $t=abc$ for the teacher output on the active sample, the output errors are $\lvert t-2\rvert$ and $\lvert t-3\rvert$, so
\[
\text{QEP is strictly worse}\iff t<\tfrac52,\qquad
\text{QEP is strictly better}\iff t>\tfrac52 .
\]
Moreover, the better of the two attains the minimum output error over all $64$ codeword triples.
\end{theorem}

\begin{proof}[Proof sketch]
Since $a>0$, the first sample is mapped to $0$ by every layer and the second stays positive; the network acts on the second sample as the scalar chain $x\mapsto abc\,x$. The uncorrected scheme rounds $a,b,c$ to $1,2,1$. QEP rounds $a$ to $1$; its second-layer target is $ab\in(\tfrac52,3)$, which rounds to $3$; its third-layer target is $abc/3\in(\tfrac12,\tfrac76)$, which rounds to $1$. The outputs are $2$ and $3$. Global optimality follows because every codeword triple outputs a non-negative integer and $2$, respectively $3$, is the nearest integer to $t$ on the corresponding side of $\tfrac52$.
\end{proof}

We state plainly what this family is. It is a scalar chain: one sample is inactive throughout and the ReLU never acts on a sign change. The threshold $t=\tfrac52$ is the midpoint of the two attainable outputs. The content of the theorem is that exact layer-wise minimization, with interior targets, unique minimizers and no clipping, produces a QEP output that is strictly worse on an open set of parameters. In the second layer QEP's choice is locally better (the pre-activation $3$ is closer to the teacher's $ab$ than $2$ is), and the third layer's grid is too coarse to exploit it. For instance, $(a,b,c)=(\tfrac75,2,\tfrac{18}{25})$ gives teacher output $2.016$, uncorrected error $0.016$ and QEP error $0.984$; changing $c$ to $\tfrac{11}{10}$ gives teacher output $3.08$, and QEP is better.

The same instance shows non-monotonicity in the propagation strength for an actual quantizer. At $(\tfrac75,2,\tfrac{18}{25})$, keeping $\alpha_1=\alpha_3=1$, with second-layer strength $\alpha_2=\tfrac12$ the selected codewords are $(1,2,1)$ and the squared output error is $4/15625$; with $\alpha_2=\tfrac34$ they are $(1,3,1)$ and it is $15129/15625$. The decision boundary is at $\alpha_2=\tfrac58$. This instance is formalized, including the tie at the boundary, where the smaller codeword is selected.

\section{Numerical illustrations}\label{sec:numerics}

The statements of Sections~\ref{sec:network} and~\ref{sec:relu} are proved. The computations of this section show how large the quantities are in small synthetic networks. They are of three kinds. For Figure~\ref{fig:gain} and Table~\ref{tab:split} the actual algorithm is run: the corrected target is computed with the Gram inverse, the weights are rounded once per run with entrywise subtractive dither on an unbounded grid, and the rounded weights are applied to all calibration samples. For the left panel of Figure~\ref{fig:relu} the rounding errors are sampled from the uniform law that this dither produces, and for the right panel sign-valued errors are enumerated; no quantizer is called in these two. Models, strengths, sample sizes and reported quantities were fixed before running. Scripts and raw outputs accompany the paper.

\paragraph{How large the gain is.} Figure~\ref{fig:gain} concerns the class of Example~\ref{ex:iso} with $d=4$ and eight calibration samples, for depths up to $64$. The simulated normalized errors agree with the closed form for a fixed teacher and for random teachers: of the $168$ simulated values, $156$ lie within two standard errors of the prediction, and the largest deviation is $2.7$ standard errors. Values at different depths come from the same runs and are not independent, so these counts are descriptive. Without compensation the error grows with the depth as $(1+\nu)^L-1$; with full compensation it stays close to $\nu$, the rounding noise of the last layer; with $\alpha=\tfrac12$ it settles at an intermediate level. At fine resolution the gain follows $10\log_{10}L$ over the depths shown. At the coarser resolutions and the largest depths shown it exceeds this value, since the uncorrected error then grows geometrically: at depth $64$ the simulated gains are $18.1$, $18.2$, $19.4$ and $26.1$~dB for $\nu=10^{-4},10^{-3},10^{-2}$ and $5\times10^{-2}$, against $10\log_{10}64\approx18.1$.

\begin{figure}[t]
\centering
\begin{tikzpicture}
\begin{axis}[width=0.5\textwidth,height=5.4cm,xmode=log,ymode=log,log basis x=2,xlabel={depth $L$},ylabel={normalized error $q_L$},
legend style={font=\scriptsize,at={(0.03,0.97)},anchor=north west,draw=none,fill=none},xtick={1,2,4,8,16,32,64},xticklabels={1,2,4,8,16,32,64},tick label style={font=\scriptsize},label style={font=\small}]
\addplot[black,thick] coordinates {(1,0.01) (2,0.0201) (3,0.030301) (4,0.040604) (6,0.0615202) (8,0.0828567) (12,0.126825) (16,0.172579) (24,0.269735) (32,0.374941) (48,0.612226) (64,0.890462)};
\addlegendentry{no compensation}
\addplot[black,only marks,mark=o,mark size=2.2pt,forget plot,error bars/.cd,y dir=both,y explicit] coordinates {(1,0.0100114) +- (0,3.13e-05) (2,0.0200965) +- (0,8.53e-05) (4,0.0405473) +- (0,0.000195) (8,0.0828237) +- (0,0.000445) (16,0.171966) +- (0,0.00107) (32,0.374935) +- (0,0.00288) (64,0.895835) +- (0,0.00963)};
\addplot[black,only marks,mark=x,mark size=2.6pt,forget plot,error bars/.cd,y dir=both,y explicit] coordinates {(1,0.00996622) +- (0,4.62e-05) (2,0.0201029) +- (0,0.000102) (4,0.0406934) +- (0,0.000212) (8,0.0829073) +- (0,0.00047) (16,0.172594) +- (0,0.00106) (32,0.374725) +- (0,0.00306) (64,0.883554) +- (0,0.01)};
\addplot[blue,thick] coordinates {(1,0.01) (2,0.0126) (3,0.013276) (4,0.0134518) (6,0.0135093) (8,0.0135132) (12,0.0135135) (16,0.0135135) (24,0.0135135) (32,0.0135135) (48,0.0135135) (64,0.0135135)};
\addlegendentry{$\alpha=\tfrac12$}
\addplot[blue,only marks,mark=o,mark size=2.2pt,forget plot,error bars/.cd,y dir=both,y explicit] coordinates {(1,0.0100216) +- (0,3.17e-05) (2,0.0126212) +- (0,4.95e-05) (4,0.0134319) +- (0,5.67e-05) (8,0.0135095) +- (0,5.63e-05) (16,0.013545) +- (0,5.67e-05) (32,0.0135087) +- (0,5.68e-05) (64,0.0135269) +- (0,5.58e-05)};
\addplot[blue,only marks,mark=x,mark size=2.6pt,forget plot,error bars/.cd,y dir=both,y explicit] coordinates {(1,0.00994729) +- (0,4.76e-05) (2,0.0125954) +- (0,6.1e-05) (4,0.0134065) +- (0,6.98e-05) (8,0.0135752) +- (0,7.74e-05) (16,0.0134728) +- (0,7.76e-05) (32,0.0135408) +- (0,8.22e-05) (64,0.0135738) +- (0,8.47e-05)};
\addplot[red,thick] coordinates {(1,0.01) (2,0.0101) (3,0.010101) (4,0.010101) (6,0.010101) (8,0.010101) (12,0.010101) (16,0.010101) (24,0.010101) (32,0.010101) (48,0.010101) (64,0.010101)};
\addlegendentry{$\alpha=1$}
\addplot[red,only marks,mark=o,mark size=2.2pt,forget plot,error bars/.cd,y dir=both,y explicit] coordinates {(1,0.00999414) +- (0,3.17e-05) (2,0.0101162) +- (0,3.34e-05) (4,0.0100873) +- (0,3.35e-05) (8,0.0101175) +- (0,3.35e-05) (16,0.0101065) +- (0,3.35e-05) (32,0.0101146) +- (0,3.36e-05) (64,0.010097) +- (0,3.38e-05)};
\addplot[red,only marks,mark=x,mark size=2.6pt,forget plot,error bars/.cd,y dir=both,y explicit] coordinates {(1,0.00996923) +- (0,4.85e-05) (2,0.0101361) +- (0,4.64e-05) (4,0.0101182) +- (0,4.71e-05) (8,0.0101191) +- (0,4.52e-05) (16,0.010034) +- (0,5.07e-05) (32,0.0100701) +- (0,5.56e-05) (64,0.0101319) +- (0,4.82e-05)};
\end{axis}
\end{tikzpicture}\hfill
\begin{tikzpicture}
\begin{axis}[width=0.5\textwidth,height=5.4cm,xmode=log,log basis x=2,xlabel={depth $L$},ylabel={gain of full compensation (dB)},
legend style={font=\scriptsize,at={(0.03,0.97)},anchor=north west,draw=none,fill=none},xtick={1,2,4,8,16,32,64},xticklabels={1,2,4,8,16,32,64},tick label style={font=\scriptsize},label style={font=\small},ymin=-1]
\addplot[gray,dashed,thick] coordinates {(1,0) (2,3.0103) (3,4.77121) (4,6.0206) (6,7.78151) (8,9.0309) (12,10.7918) (16,12.0412) (24,13.8021) (32,15.0515) (48,16.8124) (64,18.0618)};
\addlegendentry{$10\log_{10}L$}
\addplot[black,thick] coordinates {(1,0) (2,3.01008) (3,4.77121) (4,6.02082) (6,7.78216) (8,9.03199) (12,10.7938) (16,12.044) (24,13.8067) (32,15.0578) (48,16.8222) (64,18.0751)};
\addlegendentry{$\nu=10^{-4}$}
\addplot[black,only marks,mark=x,mark size=2.6pt,forget plot,error bars/.cd,y dir=both,y explicit] coordinates {(1,-0.0246986) +- (0,0.0203) (2,3.01808) +- (0,0.0237) (4,6.03606) +- (0,0.0245) (8,9.04952) +- (0,0.0297) (16,12.0254) +- (0,0.0302) (32,15.0812) +- (0,0.026) (64,18.0805) +- (0,0.0389)};
\addplot[blue,thick] coordinates {(1,0) (2,3.00813) (3,4.77121) (4,6.02277) (6,7.78803) (8,9.04176) (12,10.8114) (16,12.0695) (24,13.8478) (32,15.1146) (48,16.9105) (64,18.1949)};
\addlegendentry{$\nu=10^{-3}$}
\addplot[blue,only marks,mark=x,mark size=2.6pt,forget plot,error bars/.cd,y dir=both,y explicit] coordinates {(1,0.00102925) +- (0,0.0187) (2,3.01757) +- (0,0.0245) (4,6.00769) +- (0,0.0255) (8,9.00679) +- (0,0.0265) (16,12.0452) +- (0,0.0308) (32,15.1101) +- (0,0.0325) (64,18.1993) +- (0,0.0363)};
\addplot[red,thick] coordinates {(1,0) (2,2.98875) (3,4.77093) (4,6.04204) (6,7.84653) (8,9.13963) (12,10.9884) (16,12.3262) (24,14.2657) (32,15.696) (48,17.8255) (64,19.4525)};
\addlegendentry{$\nu=10^{-2}$}
\addplot[red,only marks,mark=x,mark size=2.6pt,forget plot,error bars/.cd,y dir=both,y explicit] coordinates {(1,-0.00130742) +- (0,0.0206) (2,2.97385) +- (0,0.024) (4,6.04419) +- (0,0.0295) (8,9.13449) +- (0,0.0293) (16,12.3555) +- (0,0.0338) (32,15.7068) +- (0,0.0426) (64,19.4054) +- (0,0.0511)};
\addplot[teal,thick] coordinates {(1,0) (2,2.90565) (3,4.76433) (4,6.12216) (6,8.10355) (8,9.57686) (12,11.7959) (16,13.5169) (24,16.261) (32,18.5451) (48,22.5194) (64,26.1531)};
\addlegendentry{$\nu=0.05$}
\addplot[teal,only marks,mark=x,mark size=2.6pt,forget plot,error bars/.cd,y dir=both,y explicit] coordinates {(1,0.0106165) +- (0,0.0177) (2,2.89823) +- (0,0.0325) (4,6.09175) +- (0,0.0325) (8,9.55486) +- (0,0.042) (16,13.5196) +- (0,0.0558) (32,18.4774) +- (0,0.0823) (64,26.1215) +- (0,0.204)};
\end{axis}
\end{tikzpicture}
\caption{Gain as a function of depth in the class of Example~\ref{ex:iso}, $d=4$, with the actual algorithm and dithered rounding of the weights. Lines are the closed form; marks are simulations with two standard errors. Left: $\nu=10^{-2}$; circles, fixed teacher $W_l=I$ ($2\times10^4$ runs); crosses, random teachers ($100$ teachers, $200$ runs each, errors clustered by teacher). Right: ratio of the two errors, random teachers, with standard errors paired by teacher. At depth one the two schemes coincide.}\label{fig:gain}
\end{figure}
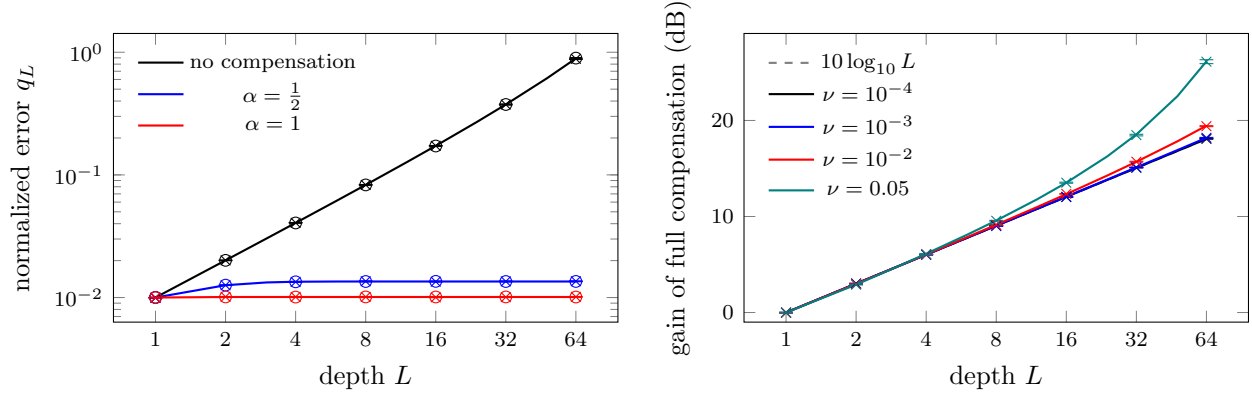

\paragraph{Where the gain comes from.} Table~\ref{tab:split} concerns a linear network with two bottlenecks, where the row spaces of the original and quantized activations differ and Corollary~\ref{cor:sched} does not apply. The gain predicted by Proposition~\ref{prop:gain}, from sources estimated on one batch of runs, agrees within two standard errors with the gain measured on an independent batch. In this network between $2\%$ and $3\%$ of the gain comes from the reduced excitation of later rounding noise, the part that a per-layer comparison does not see; the rest is error removed directly. The gain increases with the strength here, which Theorem~\ref{thm:linear} does not guarantee.

\begin{table}[t]
\centering\small
\begin{tabular}{@{}cccccc@{}}
\toprule
$\alpha$ & removed directly & through later noise & predicted gain & measured gain & error of QEP\\
\midrule
$0.25$ & $0.0882$ & $0.0017$ & $0.0898$ & $0.0897\pm0.0022$ & $0.2093$\\
$0.5$ & $0.1356$ & $0.0030$ & $0.1386$ & $0.1392\pm0.0024$ & $0.1598$\\
$0.75$ & $0.1579$ & $0.0040$ & $0.1619$ & $0.1621\pm0.0024$ & $0.1369$\\
$1$ & $0.1640$ & $0.0046$ & $0.1686$ & $0.1686\pm0.0024$ & $0.1305$\\
\bottomrule
\end{tabular}
\caption{The two parts of the gain in Proposition~\ref{prop:gain}, in a linear network $8\to6\to4\to2$ with dither of step $0.3$ and strengths $\alpha_2=\alpha_3=\alpha$. The prediction is computed from the recursions \eqref{eq:gainD}--\eqref{eq:gainL} with sources estimated on $10^4$ runs; the measured gain comes from $10^4$ independent runs and is given with two standard errors. The mean-square error without compensation is $0.2991$.}\label{tab:split}
\end{table}

\paragraph{ReLU networks.} The left panel of Figure~\ref{fig:relu} is Example~\ref{ex:noflip}. Over all runs no activation changed state, in either scheme and at either layer, and the excess of QEP over BASE divided by $t^2$ lies between $0.071$ and $0.076$, with a standard error of $0.002$, for $t$ from $5\times10^{-4}$ to $10^{-1}$, against the leading coefficient $\tfrac2{27}\approx0.074$. The harm is a second-order effect of constant relative size, present at error sizes well beyond the interval for which the lower bound is proved. The right panel is the bottleneck family described after Corollary~\ref{cor:hr}, with sign-valued errors and exact enumeration. The projections do not commute with the cells, the largest commutator norm decreases as $t$, and the bound of Theorem~\ref{thm:relu} on the excess decreases as $t^3$; relative to the uncorrected error it is about $10^{-5}$ at $t=10^{-4}$ and $3\times10^{-3}$ at $t=3\times10^{-2}$. In this family the correction helps: the error of QEP is about $28\%$ of that of BASE at every $t$, so the bound is far from attained. No activation changed state.

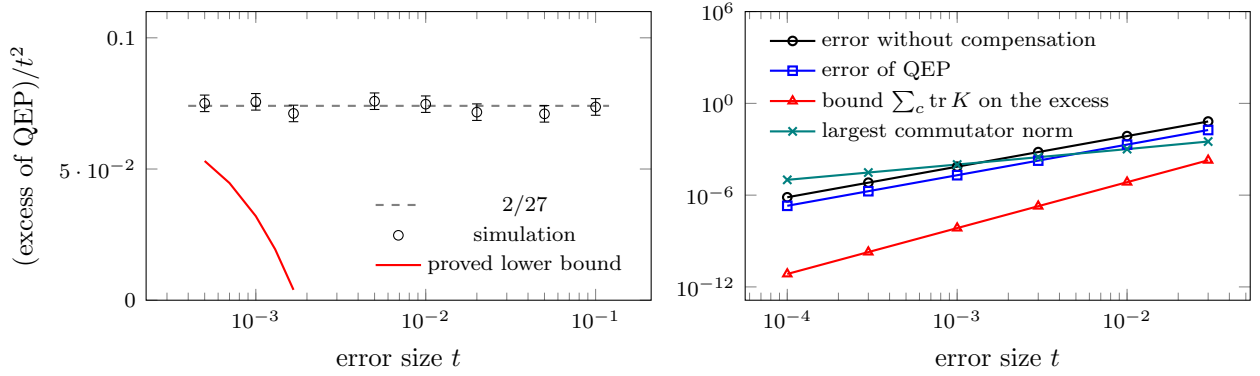
\begin{figure}[t]
\centering
\begin{tikzpicture}
\begin{axis}[width=0.5\textwidth,height=5.4cm,xmode=log,xlabel={error size $t$},ylabel={$(\text{excess of QEP})/t^2$},ymin=0,ymax=0.11,
legend style={font=\scriptsize,at={(0.97,0.05)},anchor=south east,draw=none,fill=none},tick label style={font=\scriptsize},label style={font=\small}]
\addplot[gray,dashed,thick,domain=4e-4:0.12,samples=2] {2/27};
\addlegendentry{$2/27$}
\addplot[black,only marks,mark=o,mark size=1.8pt,error bars/.cd,y dir=both,y explicit] coordinates {(0.0005,0.0750302) +- (0,0.00315) (0.001,0.0756053) +- (0,0.00315) (0.00166667,0.0711813) +- (0,0.00315) (0.005,0.0758429) +- (0,0.00315) (0.01,0.0747112) +- (0,0.00314) (0.02,0.0716695) +- (0,0.00315) (0.05,0.0710288) +- (0,0.00315) (0.1,0.0736577) +- (0,0.00317)};
\addlegendentry{simulation}
\addplot[red,thick] coordinates {(0.0005,0.0530636) (0.0007,0.0446535) (0.001,0.0320321) (0.0013,0.0194031) (0.00166667,0.00395741)};
\addlegendentry{proved lower bound}
\end{axis}
\end{tikzpicture}\hfill
\begin{tikzpicture}
\begin{axis}[width=0.5\textwidth,height=5.4cm,xmode=log,ymode=log,xlabel={error size $t$},ymax=1e6,
legend style={font=\scriptsize,at={(0.03,0.97)},anchor=north west,draw=none,fill=none,legend cell align=left},tick label style={font=\scriptsize},label style={font=\small}]
\addplot[black,thick,mark=o,mark size=1.6pt] coordinates {(0.0001,7.2565e-07) (0.0003,6.53085e-06) (0.001,7.2565e-05) (0.003,0.000653088) (0.01,0.0072569) (0.03,0.0653409)};
\addlegendentry{error without compensation}
\addplot[blue,thick,mark=square,mark size=1.6pt] coordinates {(0.0001,2.0065e-07) (0.0003,1.80585e-06) (0.001,2.0065e-05) (0.003,0.000180588) (0.01,0.0020069) (0.03,0.0180911)};
\addlegendentry{error of QEP}
\addplot[red,thick,mark=triangle,mark size=1.8pt] coordinates {(0.0001,7.01033e-12) (0.0003,1.89279e-10) (0.001,7.01034e-09) (0.003,1.8928e-07) (0.01,7.01075e-06) (0.03,0.000189381)};
\addlegendentry{bound $\sum_c\mathrm{tr}\,K$ on the excess}
\addplot[teal,thick,mark=x,mark size=2.2pt] coordinates {(0.0001,1.0002e-05) (0.0003,3.0018e-05) (0.001,0.0001002) (0.003,0.000301811) (0.01,0.00102041) (0.03,0.00319144)};
\addlegendentry{largest commutator norm}
\end{axis}
\end{tikzpicture}
\caption{ReLU networks. Left: Example~\ref{ex:noflip}, rounding errors sampled from the uniform law of subtractive dither, $10^6$ independent runs per value of $t$, two standard errors; no activation changed state in any run. Right: the bottleneck family after Corollary~\ref{cor:hr} with $\eta=\tfrac1{10}$ and sign-valued errors, exact enumeration; here the correction helps, the bound on the excess decreases as $t^3$ and the commutator as $t$.}\label{fig:relu}
\end{figure}

These are small networks chosen to exhibit the statements; they do not test whether trained networks satisfy the hypotheses. The Monte Carlo values were computed twice, with two independent implementations and random generators, one of them free of any numerical library. On identical inputs the two agree to rounding error. Run independently they agree overall: of the $168$ values of Figure~\ref{fig:gain}, $151$ differ by at most two combined standard errors and the largest difference is $3.06$; these values are correlated and the comparison is descriptive.

\section{Relation to experiments}\label{sec:empirical}

The comparisons of Sections~\ref{sec:onelayer}--\ref{sec:actual} hold the calibration input fixed; the randomness is that of rounding. A quantized model is used on other inputs, and whether an improvement obtained on calibration data carries over to them is a different question, on which those results say nothing. This section reports what we observed on both questions: published observations, a layer-wise diagnostic on two small models, and experiments on two 8B models. None of the experiments satisfies the hypotheses of our theorems, and they are not tests of them.

\subsection{Published observations}

Table~1 of \cite{qep} reports WikiText-2 perplexity for $60$ pairs (four base methods, five models, three bit widths) with and without QEP. In $53$ pairs QEP is better. In $7$ pairs it is worse, all at 2 bits: RTN on Llama-2-7B and Llama-2-13B, GPTQ on Llama-2-13B and Llama-3-8B, and AWQ on Llama-2-7B, Llama-2-70B and Mistral-7B; in all seven the perplexity is above $10^3$ both with and without QEP. These experiments use $\alpha_l=\tfrac12$ by default, except in the MLP layers of Llama-2-70B, where $\alpha_l=0$. \cite{watersic} uses the same propagated objective under the name activation drift correction, reports that it becomes ``increasingly unstable in deeper layers'' of larger models, and chooses a per-layer coefficient between drift-corrected and original calibration statistics by search; in most deeper layers of the Qwen3-8B configuration they report, the selected coefficient is close to the uncorrected end. In their reply of 30 July 2025 to reviewer pLsN in the same public discussion as above, under ``Propagation Strength'' and ``Minor Question~1'', the authors of \cite{qep} report perplexities for GPTQ at 3 bits: on Llama-2-7B, $10.881$ without QEP, $7.898$ at $\alpha=\tfrac12$ and $8.169$ at $\alpha=1$; on Llama-2-70B, $4.860$, $4.102$ and $4.771$, where the middle entry uses $\alpha=0$ in the MLP down-projection matrices and $\tfrac12$ elsewhere. The authors attribute the loss at full strength to overfitting. Perplexity is not the quantity addressed by Proposition~5.4 of \cite{qep}. We have not run these configurations and do not attribute any of these observations to the mechanism of Section~\ref{sec:identity}.

\subsection{A layer-wise diagnostic on two small models}

If, for actual rounding errors, the ratio $c/d$ of Section~\ref{sec:identity} is typically far below the threshold of Corollary~\ref{cor:iff}, the failure mode of Section~\ref{sec:counter} has little practical relevance. We measured this ratio on SmolLM2-135M \cite{smollm2} and Pythia-160M \cite{pythia}, each with two calibration texts, WikiText-2 and a public-domain novel. All linear maps inside the transformer blocks, $210$ and $48$ respectively, are first rounded to the nearest point of an unbounded uniform grid with step $\max_j\lvert W_{ij}\rvert/7$ in row $i$. On this common quantized trajectory, and for each linear map separately, we compare the layer residual with and without the correction at $\alpha=\tfrac12$, using $4096$ calibration tokens; the corrected layer is not propagated further. The correction is computed from a truncated singular value decomposition of the quantized input, keeping singular values above $10^{-10}$ times the largest, without damping; all linear maps have full effective rank except one of the $46$ maps of Pythia-160M with $d>0$, for each text. Here $c$ is computed with the rounding error of the uncorrected layer. The labels in Table~\ref{tab:diag} use the following numerical conventions, with all energies divided by the number of tokens. A linear map enters the ratios when $d>10^{-12}\max(\fro{W\delta}^2,10^{-30})$; three maps of SmolLM2-135M and two of Pythia-160M do not. With the rounding error held fixed, a map is counted as worse when $2c-(2-\alpha)d>10^{-8}\max(d,\,2\lvert c\rvert,\,\fro{B}^2,\,10^{-30})$, where $B$ is the rounding error of the uncorrected layer applied to the quantized input. After re-rounding, on calibration or held-out tokens, the comparison uses the squared Frobenius residual per token, $\fro{WX-Q\Xh}^2/N$ for the quantized weight $Q$ and $N$ tokens: a map is counted as worse when this quantity with the correction exceeds its value without the correction by more than $10^{-8}$ times the larger of the two values (or $10^{-30}$). On held-out tokens the correction fitted on the calibration tokens is reused without refitting. The protocol, the thresholds and the sample were fixed before any measurement.

\begin{table}[h]
\centering\small
\begin{tabular}{@{}lcccc@{}}
\toprule
 & \multicolumn{2}{c}{SmolLM2-135M} & \multicolumn{2}{c}{Pythia-160M}\\
 & WikiText-2 & novel & WikiText-2 & novel\\
\midrule
Linear maps with $d>0$ & $207$ & $207$ & $46$ & $46$\\
Range of $c/d$ & $[-0.063,\,0.070]$ & $[-0.065,\,0.065]$ & $[-0.012,\,0.025]$ & $[-0.038,\,0.023]$\\
Worse, same rounding error & $0$ & $0$ & $0$ & $0$\\
Worse, re-rounded, calibration & $0$ & $0$ & $0$ & $0$\\
Worse, re-rounded, held-out & $19$ & $0$ & $11$ & $12$\\
Weights outside row range & $0.19\%$ & $0.21\%$ & $9.1\%$ & $9.3\%$\\
\bottomrule
\end{tabular}
\caption{Layer-wise diagnostic at $\alpha=\tfrac12$ with deterministic rounding. ``Worse'' counts linear maps whose squared residual is larger with the correction than without it, up to the numerical tolerances stated in the text. By Corollary~\ref{cor:iff}, with the rounding error held fixed the correction is worse exactly when $c/d>\tfrac34$. The held-out row applies the correction fitted on the calibration tokens to $2048$ other tokens. The first five rows come from one measurement on the unbounded grid. The last row comes from a second measurement, fixed in advance separately, on the same weights, corrections and quantized trajectory: it is the fraction, pooled over all weights of all $210$ and $48$ linear maps, of corrected weights $W+\alpha C$, where $C=W\delta\Xh^\top\Hh^{-1}$ (with the truncated singular value inverse specified above), whose magnitude exceeds the largest magnitude of the original row.}\label{tab:diag}
\end{table}

The ratio $c/d$ never exceeds $0.07$ in absolute value and its median is below $10^{-3}$ in absolute value: the cross term is small relative to the energy of the projected inherited error, and the condition of Corollary~\ref{cor:iff} is far from violated. This ratio is not a correlation coefficient, and a small value does not mean that the two errors are uncorrelated. Re-rounding the corrected weights does not change the counts on the calibration tokens. With subtractive dither and $16$ seeds, over the $8096$ pairs of a linear map and a seed, $\lvert c/d\rvert$ stays below $0.20$; no pair is worse with the rounding error held fixed; two pairs, the same linear map of SmolLM2-135M on the novel, are worse after re-rounding on the calibration tokens; and on held-out tokens the number of linear maps that are worse ranges over the seeds from $17$ to $20$, from $0$ to $2$, and equals $11$ and $12$ in the four columns. On these models we did not observe a cancellation large enough to make the correction worse with the rounding error held fixed.

What we did observe is of the other kind. On held-out tokens the correction fitted on the calibration tokens increases the residual of a sizeable fraction of the linear maps, up to $12$ of the $48$ maps of Pythia-160M. In addition, the correction moves weights out of the range of the original row, for $9\%$ of the weights of Pythia-160M at $\alpha=\tfrac12$ and $17$--$18\%$ at $\alpha=1$, so a bounded codebook with a range set by the original weights would overload, which is the situation of Section~\ref{sec:dither}. This diagnostic intervenes on one layer at a time on a common trajectory and does not run the two schemes to the output.

\subsection{Choosing the strength on two 8B models}

\paragraph{Setting.} We quantized Qwen3-8B and Meta-Llama-3-8B with our own implementation of a QEP-style correction on top of round-to-nearest; it is not a reproduction of QEP, GPTQ, AWQ or QuIP\#. The seven linear maps of each transformer block are quantized to 4 bits with an asymmetric uniform codebook (integer codes $0$--$15$, a scale and zero point per row and group of $128$ input channels); embeddings, normalizations and the output head stay in BF16, and the dequantized weights are executed in BF16, so no speed is claimed. The maps are processed in dependency order in four groups (query--key--value, attention output, gate--up, down), maps that share an input being quantized together; each policy follows its own quantized trajectory, and the inputs of every group are taken from the actual modules. For a group with quantized input $Z$ and teacher input $X$, the correction is fitted by ridge least squares with $\lambda=\mathrm{tr}(ZZ^\top)/d$, the mean eigenvalue of the Gram matrix, and every corrected candidate is re-quantized with recomputed scales and rounded deterministically. The ridge matters for reading the results: the matrix $P_\lambda=Z^\top(ZZ^\top+\lambda I)^{-1}Z$ that replaces $P$ is not a projection, and a direction in which the Gram matrix has eigenvalue $s$ is compensated in proportion $\alpha s/(s+\lambda)$. ``Strength $1$'' below means external strength $1$ with the ridge retained; on the fitting inputs and before re-quantization, in a direction with $s=\lambda$ it compensates half of the inherited error.

For each model we drew, three times, twelve disjoint 512-token windows from WikiText-2 training text: four ($2048$ tokens) to fit the corrections and eight ($4096$ tokens) to choose strengths. Evaluation uses the first $16384$ tokens of the WikiText-2 test text and of \emph{Pride and Prejudice} (Project Gutenberg 1342, with the wrapper removed and the front matter kept), in non-overlapping 512-token windows with the context reset for each window; neither enters any decision. Each model uses its own tokenizer. The policies are: no correction (RTN); strength $\tfrac12$ in every group; strength $\tfrac12$ with the down projections uncorrected; and four rules that choose, for each group, a strength in $\{0,\tfrac12,1\}$ on the selection data, by the output error of the group (local), the error after the next block kept at RTN (block), the loss of the full network with the remaining blocks kept at RTN (network), or a first-order estimate of that loss from its gradient with respect to the block output, applied to the re-quantized candidates (gradient). The replication protocol, including all seven policies, was fixed after an exploratory round on validation text and before any of the reported evaluations.

\begin{table}[h]
\centering\small
\begin{tabular}{@{}llcccccc@{}}
\toprule
Model & Text & RTN & $\alpha=\tfrac12$ & local & block & network & gradient\\
\midrule
Qwen3-8B & WikiText-2 & $12.643$ & $12.320$ & $12.119$ & $12.199\ (1)$ & $11.970\ (2)$ & $11.870\ (3)$\\
 & novel & $12.875$ & $12.618$ & $12.665$ & $12.552\ (3)$ & $12.613\ (2)$ & $12.582\ (1)$\\
Llama-3-8B & WikiText-2 & $8.264$ & $8.031$ & $8.018$ & $8.017\ (2)$ & $8.075\ (0)$ & $8.028\ (0)$\\
 & novel & $5.786$ & $5.625$ & $5.586$ & $5.627\ (1)$ & $5.766\ (0)$ & $5.689\ (1)$\\
\bottomrule
\end{tabular}
\caption{Perplexity at 4 bits, mean over three calibration draws. In parentheses, the number of draws, out of three, in which the rule has a lower perplexity than the local rule. RTN does not depend on the draw.}\label{tab:llm}
\end{table}

\paragraph{What we observed.} Every compensating policy has a lower mean perplexity than RTN in each model and text (Table~\ref{tab:llm}), and fixed strength $\tfrac12$ is lower than RTN in each of the twelve combinations of model, text and draw; the network rule on the novel with Llama-3-8B is above RTN in one draw. No rule beats the local rule consistently. The gradient rule wins all three draws on Qwen3-8B with WikiText-2 and none on Llama-3-8B with WikiText-2. On Llama-3-8B the network rule loses to the local rule in all six comparisons, although, strength $0$ being always a candidate, none of its $816$ decisions over both models and three draws increases the full-network loss on the selection data relative to leaving that group uncorrected. These two facts compare different things, the chosen candidate with no correction in the current state and one complete policy with another; the next paragraph compares the two policies directly. The gradient rule took $4.1$--$4.8$ times the construction time of the local rule in our prototype.

\paragraph{More selection data.} To see whether the selection data were simply too few, we kept the $2048$ fitting tokens of two of the draws, enlarged the selection data in nested sets of $4096$, $16384$ and $32768$ tokens, added a separate $8192$-token diagnostic set from the same source that enters no decision, and evaluated on new windows (tokens $65536$ to $81920$) of both texts. These runs record the loss of both final policies on every data set, so the same pair can be compared with the same metric on different data. On Llama-3-8B with $4096$ selection tokens, averaged over the two draws, the network policy has a lower loss per token than the local policy on the selection data, by $0.0210$, and a higher one on the diagnostic data ($+0.0084$), on WikiText-2 ($+0.0069$) and on the novel ($+0.075$): the ordering of the two policies reverses when the data change. With more selection data the advantage on the selection data shrinks ($-0.0116$ at $16384$ tokens, $-0.0087$ at $32768$) and so does the deficit elsewhere, to $+0.0025$ on WikiText-2 and $+0.030$ on the novel, behind in every draw on the novel. On Qwen3-8B the network rule is ahead at every size, most at $16384$ tokens ($-0.033$ and $-0.018$) and less at $32768$. Selection by the output error of the current block does not help consistently either. The candidates themselves were always fitted on the same $2048$ tokens, so these results do not separate the quality of the candidates from the accuracy of choosing among them.

\paragraph{Where the overestimate lies.} Measuring each policy against RTN separates the two rules. For a policy, take its gain over RTN on the selection data minus its gain over RTN on the diagnostic data; the difference between the two policies' values is the change of their ordering reported above. With $4096$ selection tokens, averaged over the two draws, this overestimate is $+0.022$ on Llama-3-8B and $+0.034$ on Qwen3-8B for the network rule, and $-0.007$ and $-0.006$ for the local rule. A further run with the fitting data of the first draw, two new selection draws and $32$ source windows used by no policy gave the same ordering in all four cases, with $0.029$ to $0.049$ for the network rule and $0.008$ to $0.022$ for the local rule. The network rule chooses by the loss of the full network on the selection data, which is also the quantity reported on those data; the local rule chooses by a linear output error. A small standard error does not exclude the overestimate: on Llama-3-8B in the first draw the network rule is ahead of the local rule on the selection data by $0.025$, about six standard errors over the eight windows, and behind on the diagnostic data. The overestimate is comparable in the two models; the gains are not. On the selection data the network rule is ahead of the local rule by $0.052$ on Qwen3-8B and by $0.021$ on Llama-3-8B, and only the first survives. Fixed strength $\tfrac12$ makes no choice on the selection data and pays no such cost. Reusing the same eight windows for every decision is not the explanation. In the same run, giving each decision eight windows not used before in its path, with everything else fixed, did not reduce the overestimate of the network rule: summed over the decisions and measured on a fixed set of eight held-out source windows, it went from $0.031$--$0.043$ with reused windows to $0.070$--$0.083$ with fresh ones, and the mean final loss on held-out source windows did not improve. Fresh windows remove reuse across decisions, but not the choice, at each decision, of the candidate that scores best on the very windows on which its gain is then measured. This run used a protocol, predictions and a stopping rule fixed before inference; fresh windows also change the amount of text seen, so it does not isolate reuse at an equal data budget.

\paragraph{Changing the suffix.} Choosing group by group, each rule above scores a candidate with the later blocks still at RTN, and then changes those blocks. To see whether this explains the failures, we replayed $48$ recorded decisions (two draws per model; three blocks per draw, four groups each; the groups of a fourth, final block serve as a control) with the earlier state and the three candidates fixed, scoring them once with the RTN suffix and once with the corrected suffix reconstructed from the recorded decisions. In $11$ of the $48$ the preferred candidate on the selection data reverses by more than $10^{-4}$ in loss; none of these $11$ reversals recurs, for the same pair and in the same direction, on the diagnostic data. Dependence on the suffix exists, but we found no evidence that it accounts for the losses on new text.

\paragraph{The strength from the fixed-error quadratic.} Remark~\ref{rem:alphastar} warns that the minimizer of a quadratic with the rounding error held fixed is not by itself a rule for an actual quantizer. We measured how far it goes at 2 bits (codes $0$--$3$, same group structure), in $48$ slices (two draws per model, three blocks spread over the depth, four groups), each candidate re-quantized with recomputed scales and applied to a common RTN trajectory. With $S$ the output of the ridge correction, the strength minimizing $\fro{R_0-\alpha S}^2$ on the selection data, clipped to $[0,1]$, ranks the re-quantized strengths with a mean rank correlation of $0.977$ to $0.994$ across models and draws, and on the diagnostic data it lowers the normalized residual of the group by $6.4$ to $9.2\%$ relative to strength $\tfrac12$, on average within each model and draw (in $47$ of the $48$ slices). Seven of the $48$ values were above $1$ before clipping, the largest $1.16$; since the ridge has already shrunk the correction, these numbers say nothing about the best unshrunk strength. The fixed pair $\{\tfrac12,1\}$, chosen between on the selection data with strength $0$ retained as a fallback (as in the grid of ten non-zero strengths), already retains $98.2$ to $99.6\%$ of the diagnostic gain of a ten-point grid. These are linear residuals, not perplexities. The scale is recomputed for each candidate, which places this setting outside Remark~\ref{rem:piecewise}; a high rank correlation shows that the quadratic predicts the ordering, not that the loss is a smooth function of the strength.

\subsection{What the experiments do and do not show}

On a fixed calibration input, Sections~\ref{sec:network}--\ref{sec:actual} give conditions under which compensation cannot hurt and mechanisms by which it can. In the experiments, compensation helped on average relative to no compensation, and on the small models the cancellation of Section~\ref{sec:counter} was not observed with the rounding error held fixed; we did not measure the cross term on the 8B models, so we neither confirm nor exclude it there. The held-out comparisons reveal a further limitation: gains measured on fitting or selection data need not persist on other inputs. A correction that lowers a layer residual on calibration tokens raises it on others, and a policy that is better than another on the selection data can be worse on new text. This is separate from the two increases on calibration tokens observed after re-rounding with dither in the small-model diagnostic, and from the fixed-input failures of Sections~\ref{sec:counter} and~\ref{sec:actual}. We have not identified the cause of the held-out losses. The overestimate is larger for the rule that chooses by the reported quantity, and reusing the selection data across decisions does not explain it; the finite fitting data, the size and domain of the selection data, the ridge and the interaction between layers are not separated by these experiments. Calibrating quantized weights against the output of a whole block rather than a single layer is not new \cite{brecq,omniquant}; our network rule goes further and scores candidates by the loss of the full network. What these experiments add is a measurement, in one setting, of how far such choices carry over to new text. The scope is narrow: two models, two or three draws, two texts that may appear in pretraining data, windows rather than independent documents, our own RTN-backed implementation and dequantized execution.

\section{Formal verification}\label{sec:lean}

The statements below are checked in Lean~4 (version 4.34.1) with Mathlib at commit \texttt{d13f23b}. The accompanying artifact contains 96 modules with 969 explicitly named theorem declarations, auxiliary lemmas included; every one of them depends only on the axioms \texttt{propext}, \texttt{Classical.choice} and \texttt{Quot.sound}. The number of declarations is not a count of results. The artifact also contains positive controls and ten deliberately false statements that are required to fail. The table lists the eleven modules for the statements of Sections~\ref{sec:identity}, \ref{sec:counter} and~\ref{sec:actual}, with 153 declarations; the remaining modules are described after it.

\begin{center}\small
\begin{tabular}{@{}lp{10.2cm}r@{}}
\toprule
Module & Content & Decl. \\
\midrule
\texttt{QepSharedE} & Theorem~\ref{thm:counter}, the sign remark, arbitrarily small $\varepsilon$ & 15 \\
\texttt{QepSharedEExtended} & Exact parameter domain of Theorem~\ref{thm:counter}: the quantized weights are non-zero for every $\varepsilon>0$ except $\varepsilon=1$ and $\varepsilon^2+\varepsilon=1$ & 6 \\
\texttt{QepAlpha} & Theorem~\ref{thm:alpha}; validity and monotonicity of the bound of \cite[Cor.~B.6]{qep} & 8 \\
\texttt{QepProjectionExample} & First part of Example~\ref{ex:proj} & 7 \\
\texttt{QepMatrixFamily} & Second part of Example~\ref{ex:proj} for $D=\varepsilon I$ and general inputs, and the instance $X=[\,I_2\mid\mathbf 1\,]$ & 25 \\
\texttt{QepCorrected} & Proposition~\ref{prop:identity}, Corollary~\ref{cor:iff}(i), finite weighted form of Theorem~\ref{thm:meansq} & 10 \\
\texttt{QepCorollary} & Corollary~\ref{cor:iff}(ii)--(iv) & 12 \\
\texttt{QepGram} & $P=\Xh^\top\Hh^{-1}\Xh$ is a symmetric idempotent with $\Xh P=\Xh$ when $\Xh$ has full row rank; Corollary~\ref{cor:iff}(i) for the actual QEP residuals; the decomposition of Remark~\ref{rem:local} & 13 \\
\texttt{QepDitherFinite} & Proposition~\ref{prop:dither} for independent finitely supported errors with zero mean and equal second moments & 6 \\
\texttt{QepInterior} & Theorem~\ref{thm:actual} and its two numerical instances & 43 \\
\texttt{QepInteriorAlpha} & The $\alpha_2$ instance of Section~\ref{sec:actual}, with its switch at $\tfrac58$ & 8 \\
\bottomrule
\end{tabular}
\end{center}

In \texttt{QepSharedE} the residuals and the bound $U$ are computed from the recursion of Section~\ref{sec:setting}, not assumed. \texttt{QepInterior} defines both encoders by an explicit four-way minimization and derives the selected codewords.

The remaining modules were reviewed independently of their authors. They cover the following, each within the stated scope.
\begin{itemize}[leftmargin=*,itemsep=1pt]
\item Theorem~\ref{thm:linear} for finitely supported error laws, with the actual Gram matrices, the expectation step, the joint induction on the two moments, and layers of zero strength.
\item Example~\ref{ex:relu} in full: the recursion with continuous dither, both integrals, the strict inequality and the probability of a positive output.
\item For the actual continuous dither on a fixed finite grid: the distribution of the saturating quantizer, the two bounded-codebook examples of Section~\ref{sec:dither}, and per-layer energy and cross-term statements on a common input with an arbitrary history, under explicit integrability and rank assumptions.
\item Rounding with feedback, for the actual sequential process with continuous dither: with parameters fixed in advance, the relation between final errors and innovations, the covariance of the final error, the output energy on an unbounded grid and the overload term for a finite codebook; with measurable parameters depending on earlier randomness, the joint law of the standardized innovations only.
\item Examples~\ref{ex:traj} and~\ref{ex:composite} and the repaired composite bound; the non-commutative family of Example~\ref{ex:proj}, with the spectral norm of $D$.
\item For Theorem~\ref{thm:relu}, algebra only: the cell-wise identities; the two inequalities \eqref{eq:comm} and the two matrix bounds derived from them, under a bilinear bound on the commutator that is not identified with the spectral norm; and an order-propagation step with defect terms in which the per-layer inequalities are hypotheses. These are not linked to the actual Gram projections, the activations or the expectations.
\item Proposition~\ref{prop:commute}; the schedule example and the example for the bounds in Section~\ref{sec:network}.
\end{itemize}

Not formalized: Theorem~\ref{thm:meansq} as an abstract conditional-expectation statement (the finite weighted version and the fixed-grid versions above are formalized); Theorem~\ref{thm:linear} for general error laws, where the integrability argument near singular Gram matrices is a paper proof; Proposition~\ref{prop:gain}, Corollary~\ref{cor:sched} and Example~\ref{ex:iso}, of which the first two were checked by exact enumeration on small networks and all three by the simulations of Section~\ref{sec:numerics}; the distributional property of subtractive dither on an unbounded grid for a general lattice, which we cite, and Proposition~\ref{prop:dither} for continuous dither; Theorem~\ref{thm:relu} as a statement about the actual stochastic network, Corollary~\ref{cor:hr}, the perturbation bound for projections and the weighted form of Proposition~\ref{prop:commute}; the expectation in Example~\ref{ex:noflip}, which is proved on paper in Appendix~\ref{app:noflip}; the energy statements for rounding with feedback when the parameters depend on earlier randomness; Remark~\ref{rem:piecewise}; the consequence drawn in Remark~\ref{rem:local} for an exact discrete minimizer; and the matrix family of Example~\ref{ex:proj} for general diagonal $D$. Verification consists of standard elaboration and kernel checking with printed axiom dependencies; no independent proof checker was run.

\section{Limitations and open questions}\label{sec:open}

\begin{itemize}[leftmargin=*,itemsep=2pt]
\item The theorems compare the two schemes on a fixed calibration input. They say nothing about new inputs, where the largest losses we observed occur (Section~\ref{sec:empirical}); the overestimate we measured is larger for the rule that selects by full-network loss, but reuse of the selection data does not explain it and we have not established the cause. The cross term $c/d$ was measured only on two small models with single-layer interventions; how often $2c>(2-\alpha)d$ occurs in larger models, at lower bit widths or along each scheme's own trajectory is not known to us.
\item The composite inequality $\fro{\delta_{L+1}^{\qep}}\le U^{\base}$, also asserted in \cite[Theorem~B.5]{qep}, fails as an exact statement (Example~\ref{ex:composite}) and holds to first order, in the repaired form given in Appendix~\ref{app:more}.
\item At the network level, Theorem~\ref{thm:linear} covers linear networks and Example~\ref{ex:relu} shows that the mean-square comparison fails for ReLU networks. Theorem~\ref{thm:relu} bounds the cost of non-commuting projections for ReLU networks but requires stable activations, which we have not established for any trained network; a statement that tolerates a small probability of activation changes is open, and other activations are not treated. How $\E(U^{\qep})^2$ and $\E(U^{\base})^2$ compare for linear networks under uniform dither is not settled in this paper.
\item We have not analyzed damping, although the experiments of Section~\ref{sec:empirical} use it. With a ridge term, $T-X=(Z-X)(I-\alpha P_\lambda)$ still contracts the inherited error on the fitting inputs, before re-quantization, but the inequality for the activation moment in the proof of Theorem~\ref{thm:linear} does not carry over as stated, and a network comparison is not established. With a ridge term the matrix playing the role of $P$ is not idempotent and Proposition~\ref{prop:identity} must be replaced by the general expansion $\fro{A-B-\alpha S}^2-\fro{A-B}^2=\alpha^2\fro{S}^2-2\alpha\ip{A-B}{S}$ for the correction direction $S$.
\item The counterexamples to the shared-perturbation theorem are small. Example~\ref{ex:proj} gives a matrix family and a non-commutative family; we have not looked for counterexamples with large or structured weights.
\end{itemize}

\section{Related work}

RTN, GPTQ \cite{gptq}, AWQ \cite{awq} and QuIP \cite{quip} are the layer-wise methods to which QEP is applied in \cite{qep}; QuIP\# \cite{quipsharp} extends QuIP with lattice codebooks. GPTAQ \cite{gptaq} adopts an asymmetric calibration objective of the same form and is the method from which the authors of \cite{qep} distinguish their work by the guarantee examined here. Qronos \cite{qronos}, NestQuant \cite{nestquant} and WaterSIC \cite{watersic} use the same propagated objective, and \cite{rescomp} studies the residual introduced by compensation itself. Error analyses that account for the quantization of earlier layers exist: SPFQ \cite{spfq} gives full-network error bounds for a stochastic path-following quantizer, and \cite{qronostheory} gives error bounds for OPTQ and Qronos and compares the two bounds. These are high-probability or worst-case bounds for specified algorithms; they do not assert that, realization by realization or for the whole network, correction is never worse than no correction, which is the statement examined here. For a single layer, the information-theoretic limits of quantizing for a matrix product are characterized in \cite{optqmm}, with nested-lattice constructions, and under an i.i.d.\ Gaussian weight model at high rate \cite{watersic} reports a rate gap of about $0.255$ bit to the information-theoretic limit. The independence of dithered quantization error from the input goes back to \cite{schuchman}; see \cite{graystockham,zamirfeder}. Subtractive dither has been applied to the compression of pretrained network weights \cite{choi}, and stochastic rounding to low-precision training \cite{gupta}; neither work compares corrected and uncorrected layer-wise quantization. Treating active and inactive ReLU samples differently in layer-wise reconstruction, with a cascade analysis, goes back to Net-Trim \cite{nettrim}. Quantizing a residual after predicting from an already quantized past is the structure of predictive coding and noise shaping, where the effect of noisy prediction and the limits on feedback gain are classical \cite{zke,oe}. Closer still are the classical analyses of finite-word-length effects. In roundoff-noise analysis of digital filters the output noise is computed through Gramians and minimized over realizations \cite{mullisroberts}, and error feedback is designed to reduce it; the noise gain is then a quadratic function of the feedback, with a gain term and a cross term and an optimal strength \cite{hinamoto}. The per-layer identity of Section~\ref{sec:identity} has this form and is elementary. The second-moment recursion of the uncorrected scheme is that of a linear system with multiplicative noise, and the propagation of orderings by positive maps is standard in that theory \cite{bamiehfilo}. The theory of dither is surveyed in \cite{lipshitz}. What is specific here is that the correction is a projection determined by the realized activations of the corrected scheme, so that it is neither a fixed feedback filter nor a fixed state transition, and that two trajectories are compared; the step to be proved is that this correction reduces both the error moment and the moment that excites later noise. That error feedback can amplify correlated errors is therefore not new; what we establish is a joint comparison of error and activation moments along two realized trajectories, its exact gain, and the limits imposed by nonlinear geometry and finite codebooks.

\section{Conclusion}

Under the hypotheses of Theorem~\ref{thm:linear}, compensating inherited quantization error with strengths fixed in advance does not increase the mean-square output error of a linear network. The proof compares the schemes along their own trajectories through both error and activation moments; the exact gain separates direct error removal from reduced excitation of later rounding noise. Stronger compensation is no worse under additional row-space conditions, with a closed-form depth gain in a specified model. Outside that setting we exhibit several ways in which the comparison fails. For ReLU networks, stable activations alone are insufficient; commutators between calibration projections and activation-pattern partitions bound the excess. For fixed perturbations, the per-layer comparison depends on the error cross term. With actual finite codebooks, even unique exact layer-wise minimizers without overload can yield a larger output error. Local improvement is therefore not a network guarantee. All of this holds the calibration input fixed; on two 8B models, the harder part was carrying an improvement on calibration data over to new text.

\paragraph{Use of AI tools.} Claude, ChatGPT and Codex assisted in developing counterexamples, parts of the analysis, numerical and Lean code, and manuscript revisions. The mathematical arguments are given in the paper and its appendices; Section~\ref{sec:lean} states the scope and limitations of the formal checks.

\appendix
\section{Proof for Example~\ref{ex:noflip}}\label{app:noflip}

% Requires amsmath, amssymb, amsthm; \fro{.} is the Frobenius norm, e.g. \newcommand{\fro}[1]{\|#1\|_F}
\paragraph{Setting.}
Columns are samples. Let
\[
X_1=\begin{pmatrix}1&-1&1\\-1&1&1\end{pmatrix},\qquad W_1=I_2,\qquad
W_2=\begin{pmatrix}1&-2\\-1&-1\end{pmatrix},
\]
so that $X_2=\mathrm{ReLU}(W_1X_1)=\begin{pmatrix}1&0&1\\0&1&1\end{pmatrix}$,
$W_2X_2=\begin{pmatrix}1&-2&-1\\-1&-1&-2\end{pmatrix}$ and
$X_3=\mathrm{ReLU}(W_2X_2)=\begin{pmatrix}1&0&0\\0&0&0\end{pmatrix}$.
Fix $t>0$ and let $E_1=tH$, $E_2=tG$, where the eight entries of $H=(h_{ij})$, $G=(g_{ij})$ are independent and uniform on $[-1,1]$.
Put $Q_1=W_1+E_1$, $Z_2=\mathrm{ReLU}(Q_1X_1)$ and
\[
Q_2^{\mathrm B}=W_2+E_2,\qquad
Q_2^{\mathrm Q}=W_2+W_2(X_2-Z_2)Z_2^\top(Z_2Z_2^\top)^{-1}+E_2 .
\]
For either scheme $Y=\mathrm{ReLU}(Q_2Z_2)$ and $L=\fro{Y-X_3}^2$. Let
$P_t=Z_2^\top(Z_2Z_2^\top)^{-1}Z_2$, $w=(1,-2)$, $e_1=(1,0,0)^\top$,
\[
K=\begin{pmatrix}h_{11}-h_{12}&0&h_{11}+h_{12}\\0&-h_{21}+h_{22}&h_{21}+h_{22}\end{pmatrix},
\]
\[
b=h_{11}-h_{12},\qquad c_t=wK(I-P_t)e_1,\qquad c_0=\tfrac13(-2h_{12}+4h_{21}).
\]

\begin{theorem}\label{thm:relu-counterexample}
\begin{enumerate}
\item[(a)] For $0<t\le\frac1{10}$ and every realization of $(H,G)$: all layer-1 pre-activations and all layer-2 pre-activations of both schemes have the strict signs of the teacher's; $Z_2=X_2+tK$; $\sigma_{\min}(Z_2)\ge 1-\sqrt8\,t\ge 1-3t>0$; and $Q_1$, $Q_2^{\mathrm B}$, $Q_2^{\mathrm Q}$ are invertible.
\item[(b)] For $0<t\le\frac1{10}$, $\ \mathbb E L_{\mathrm{QEP}}-\mathbb E L_{\mathrm{BASE}}=t^2\,\mathbb E(c_t^2-b^2)$, with $\mathbb E b^2=\frac23$ and $\mathbb E c_0^2=\frac{20}{27}$.
\item[(c)] For $0<t\le\frac1{10}$ and every $H$: $\|P_t-P_0\|_2\le 4t(1+t)$, $|c_t|\le 6$, $|c_0|\le 2$, $|c_t-c_0|\le 24t(1+t)$.
\item[(d)] For $0<t\le\frac1{10}$,
\[
\mathbb E L_{\mathrm{QEP}}-\mathbb E L_{\mathrm{BASE}}\;\ge\;t^2\Bigl(\tfrac{2}{27}-42\,t(1+t)\Bigr),
\]
and the right-hand side is strictly positive for $0<t\le\frac1{600}$.
\end{enumerate}
\end{theorem}

\begin{proof}
\emph{Step 1 (layer-1 gates).} $Q_1X_1=X_1+tHX_1$ and row $i$ of $HX_1$ is $(h_{i1}-h_{i2},\,-h_{i1}+h_{i2},\,h_{i1}+h_{i2})$, with entries of modulus at most $2$. The entries of $X_1$ are $\pm1$, so no sign changes when $2t<1$, and $Z_2=X_2+tK$. Write
$k_{11}=h_{11}-h_{12}$, $k_{13}=h_{11}+h_{12}$, $k_{22}=h_{22}-h_{21}$, $k_{23}=h_{21}+h_{22}$ (so $|k_{ij}|\le2$) and
$a=1+tk_{11}$, $c=1+tk_{13}$, $b'=1+tk_{22}$, $d=1+tk_{23}$; then
$Z_2=\begin{pmatrix}a&0&c\\0&b'&d\end{pmatrix}$ with $a,b',c,d\in[1-2t,1+2t]$.

\emph{Step 2 (Gram matrix).} $\fro{K}^2=2\sum_{i,j}h_{ij}^2\le 8$. Since $X_2X_2^\top=\begin{pmatrix}2&1\\1&2\end{pmatrix}$ has eigenvalues $1,3$, $\sigma_{\min}(X_2)=1$, and Weyl's inequality gives $\sigma_{\min}(Z_2)\ge1-t\|K\|_2\ge1-\sqrt8\,t$. Hence $Z_2Z_2^\top$ is invertible and $\|Z_2^\top(Z_2Z_2^\top)^{-1}\|_2\le(1-\sqrt8\,t)^{-1}$.

\emph{Step 3 (layer-2 pre-activations).} $Q_2^{\mathrm B}Z_2=W_2X_2+tW_2K+tGZ_2$, and
\[
Q_2^{\mathrm Q}Z_2=W_2Z_2+W_2(X_2-Z_2)P_t+tGZ_2=W_2X_2+tW_2K(I-P_t)+tGZ_2 .
\]

\emph{Step 4 (rows of $W_2K$).} The rows are $r_1=wK=(k_{11},-2k_{22},k_{13}-2k_{23})$ and $r_2=(-k_{11},\allowbreak-k_{22},\allowbreak-k_{13}-k_{23})$. The functions $\|r_i\|_2^2$ are convex in $H$, so they are maximized at vertices of the cube, where $(k_{11},k_{13})\in\{(\pm2,0),(0,\pm2)\}$ and independently $(k_{22},k_{23})\in\{(\pm2,0),(0,\pm2)\}$. Enumerating the sign choices gives $\|r_1\|_2^2\le36$ and $\|r_2\|_2^2\le16$ at every vertex, so
\[
\|wK\|_2\le 6,\qquad \|r_2\|_2\le4 .
\]

\emph{Step 5 (layer-2 gates).} The $(i,j)$ entry of $W_2K(I-P_t)$ is $r_i(I-P_t)e_j$, of modulus at most $\|r_i\|_2\le6$; the same bound holds for $W_2K$. The $(i,j)$ entry of $GZ_2$ has modulus at most $z_{1j}+z_{2j}\le2+4t$. Thus both pre-activations differ from $W_2X_2$ entrywise by at most $t(8+4t)\le0.84<1$ for $t\le\frac1{10}$, whereas every entry of $W_2X_2$ has modulus at least $1$.

\emph{Step 6 (invertibility).} $\sigma_{\min}(Q_1)\ge1-t\fro{H}\ge1-2t>0$. The eigenvalues of $W_2W_2^\top=\begin{pmatrix}5&1\\1&2\end{pmatrix}$ are $\frac12(7\pm\sqrt{13})$, so $\sigma_{\min}(W_2)\ge1.3$ and $\|W_2\|_2\le2.305$; hence $\sigma_{\min}(Q_2^{\mathrm B})\ge1.3-2t>0$. Finally $Q_2^{\mathrm Q}=W_2-tW_2KZ_2^\top(Z_2Z_2^\top)^{-1}+tG$, and by Step~2 the perturbation has norm at most $t\bigl(2.305\sqrt8/(1-\sqrt8\,t)+2\bigr)\le1.11<1.3$ for $t\le\frac1{10}$. This proves (a).

\emph{Step 7 (the loss is one scalar).} By Step~5, $Y$ has the support of $X_3$, so $L=(Y_{11}-1)^2$ with $Y_{11}$ the $(1,1)$ pre-activation. As $z_{21}=0$ and $z_{11}=a$, Step~3 gives
\[
Y^{\mathrm B}_{11}-1=t\,(b+g_{11}a),\qquad Y^{\mathrm Q}_{11}-1=t\,(c_t+g_{11}a).
\]

\emph{Step 8 (cancellation of the last-layer error).} $b$, $c_t$, $a$ depend on $H$ only and $g_{11}$ is independent of $H$ with mean zero, so for $\xi\in\{b,c_t\}$,
$\mathbb E(\xi+g_{11}a)^2=\mathbb E\xi^2+\tfrac13\mathbb E a^2$. Subtracting gives the identity in (b).

\emph{Step 9 ($t=0$).} The orthogonal complement of the row space of $X_2$ is spanned by $m_0=(1,1,-1)^\top$, so $I-P_0=\frac13m_0m_0^\top$ and
$c_0=\frac13(k_{11}-2k_{22}-k_{13}+2k_{23})=\frac13(-2h_{12}+4h_{21})$. With $\mathbb Eh_{ij}^2=\frac13$: $\mathbb Eb^2=\frac23$, $\mathbb Ec_0^2=\frac{4+16}{27}=\frac{20}{27}$, and $|c_0|\le2$.

\emph{Step 10 (projection bound).} $m=(cb',\,ad,\,-ab')^\top$ is nonzero and orthogonal to both rows of $Z_2$, so $I-P_t=mm^\top/\|m\|_2^2$. For unit vectors $u,v$ spanning lines at angle $\theta\in[0,\frac\pi2]$, $\|uu^\top-vv^\top\|_2=\sin\theta$; therefore
\[
\|P_t-P_0\|_2=\sin\angle(m,m_0)=\frac{\operatorname{dist}(m_0,\operatorname{span}\{m\})}{\|m_0\|_2}\le\frac{\|m-m_0\|_2}{\sqrt3}.
\]
Each coordinate of $m-m_0$ is $\pm\bigl((1+t\alpha)(1+t\beta)-1\bigr)$ with $|\alpha|,|\beta|\le2$, of modulus at most $4t+4t^2$. Hence $\|P_t-P_0\|_2\le4t(1+t)$.

\emph{Step 11.} $|c_t|\le\|wK\|_2\,\|(I-P_t)e_1\|_2\le6$ and $|c_t-c_0|=|wK(P_0-P_t)e_1|\le6\cdot4t(1+t)=24t(1+t)$. This proves (c).

\emph{Step 12 (lower bound).} With $\delta=c_t-c_0$, $c_t^2-c_0^2=2c_0\delta+\delta^2\ge-2|c_0||\delta|$, so
\[
\mathbb Ec_t^2-\mathbb Ec_0^2\ge-48\,t(1+t)\,\mathbb E|c_0|\ge-48\,t(1+t)\sqrt{20/27}\ge-42\,t(1+t),
\]
because $\sqrt{20/27}\le\frac78$. By Steps~8--9,
$\mathbb EL_{\mathrm{QEP}}-\mathbb EL_{\mathrm{BASE}}=t^2\bigl(\frac2{27}+\mathbb Ec_t^2-\mathbb Ec_0^2\bigr)\ge t^2\bigl(\frac2{27}-42t(1+t)\bigr)$.
Since $t\mapsto42t(1+t)$ is increasing and equals $\frac{42\cdot601}{360000}<0.0702<\frac2{27}$ at $t=\frac1{600}$, the bound is positive on $(0,\frac1{600}]$.
\end{proof}

\begin{remark}
(i) Since $4t(1+t)\le12t$ and $6\le\sqrt{40}$, the cruder estimates $\|P_t-P_0\|_2\le12t$, $|c_t^2-c_0^2|\le960t$ and positivity for $0<t\le10^{-5}$ also hold.
(ii) Apart from Steps~8 and~9 the argument is deterministic and uses only $|h_{ij}|,|g_{ij}|\le1$; Step~8 uses that $G$ is independent of $H$ with $\mathbb Eg_{11}=0$, and Step~9 uses the second moments of the $h_{ij}$. If the $h_{ij}$ are i.i.d.\ with mean zero and second moment $s$, the leading coefficient is $\frac{20s}{9}-2s=\frac{2s}{9}$ (so $\frac29$ for symmetric signs $\pm1$).
(iii) $c_t(H)=-c_{-t}(-H)$, so for a symmetric law $t\mapsto\mathbb Ec_t^2$ is even and the true correction is $O(t^2)$; numerically $\mathbb E(c_t^2-b^2)\approx\frac2{27}-0.283\,t^2$ and stays positive on all of $(0,\frac1{10}]$; this is a numerical observation. The interval $(0,\frac1{600}]$ is the proved one and is not sharp.
\end{remark}

\section{Further counterexamples in the shared-perturbation model}\label{app:more}

\begin{example}[Not an artifact of $P=I$]\label{ex:proj}
Take one output, one input dimension and two calibration samples: $W=1$, $\Xh=(1,0)$, $X=(2,1)$, $E=1$, $\alpha=1$. Then $A:=W\delta=(1,1)$, $P=\mathrm{diag}(1,0)$, $B:=E\Xh=(1,0)$. The QEP residual is $A(I-P)-B=(-1,1)$ with squared norm $2$; the BASE residual on the same input is $A-B=(0,1)$ with squared norm $1$.

The counterexample of Theorem~\ref{thm:counter} also lifts to matrices. Let $d\ge2$, $L=2$, identity activations, $W_1=W_2=I_d$, $X=[\,I_d\mid\mathbf 1\,]\in\R^{d\times(d+1)}$, $\alpha_1=\alpha_2=1$, and $E_1=D$, $E_2=-D$ with $D=\mathrm{diag}(\varepsilon_1,\dots,\varepsilon_d)$, $0<\varepsilon_i<\tfrac12$. Both schemes use $\Wh_1=I+D$ and $\Xh_2=(I+D)X$, which has full row rank, and $P_2\neq I_{d+1}$. The correction is $-D(I+D)^{-1}$, so $\Wh_2^{\base}=I-D$ and $\Wh_2^{\qep}=(I+D)^{-1}-D$, all invertible, and
\[
R_1=-DX,\qquad R_2^{\base}=D^2X,\qquad R_2^{\qep}=(D+D^2)X .
\]
Hence $U^{\qep}>U^{\base}$; for $D=\varepsilon I$ the gap is exactly $\varepsilon\fro{X}$. The mechanism is the scalar one applied coordinate by coordinate.

A family that does not reduce to independent scalar channels: let $N=\bigl(\begin{smallmatrix}0&1\\0&0\end{smallmatrix}\bigr)$, $D=\mathrm{diag}(a,b)$ with $a>2b>0$, $W_1=I$, $W_2=D$, $X=I$, $\alpha=(1,1)$, and $E_1=tN$, $E_2=t(ND-DN)=(b-a)tN$ with $t>0$. Both schemes have $\Xh_2=I+tN$. The outputs are $D+btN$ for BASE and $D+(b-a)tN$ for QEP, so the output errors are $bt$ and $(a-b)t$, and $U^{\base}=(a+b)t$, $U^{\qep}=(2a-b)t$. All weights and Gram matrices are invertible for every $t$. Here $E_2$ is a commutator; when $a=b$ it vanishes, and QEP is exact while BASE has error $bt$.
\end{example}

\begin{example}[The trajectory substitution fails independently]\label{ex:traj}
Take $L=3$, scalars, identity activations, $W=(1,1,1)$, $X=1$, $\alpha=(1,1,1)$ and $E(t)=t\,(\tfrac12,-\tfrac1{10},-\tfrac7{27})$ with $0<t\le1$. All quantized weights are positive. On QEP's own inputs, the QEP residual is no larger in magnitude than the BASE residual evaluated on the same inputs, at every layer: the per-layer comparison on a common input holds throughout. Yet on each scheme's own trajectory the third-layer residuals satisfy
\[
\lvert R_3^{\qep}\rvert-\lvert R_3^{\base}\rvert=\frac{t\,(64+69t-14t^2)}{540}>0 ,
\]
with a positive first-order coefficient $16/135$. At $t=1$ the BASE residual is $0$ and the QEP residual is $119/540$. A layer-wise comparison on common inputs therefore does not transfer to the two actual trajectories, even to first order.
\end{example}

Theorem~B.5 of \cite{qep} also asserts the composite inequality $\fro{\delta_{L+1}^{\qep}}\le U^{\base}$, which compares the actual QEP output error with the BASE bound. In Theorem~\ref{thm:counter} it holds with equality. As an exact statement it fails as well, although not to first order in the example we know.

\begin{example}[The composite inequality fails for finite perturbations]\label{ex:composite}
Take $L=3$, scalars, identity activations, $W=(1,1,1)$, $X=1$, $\alpha=(1,1,1)$, and $E=(-\tfrac12,-\tfrac12,3)$. The BASE weights are $(\tfrac12,\tfrac12,4)$, with output $1$ and residuals $(\tfrac12,\tfrac34,0)$, so $U^{\base}=\tfrac54$. The QEP weights are $(\tfrac12,\tfrac32,\tfrac{13}{3})$, with output $\tfrac{13}{4}$ and output error $\tfrac94>U^{\base}$. Along the ray $E(t)=t(-\tfrac12,-\tfrac12,3)$, $0<t\le1$, the difference between the QEP output error and $U^{\base}$ equals $t(3t-1)/2$, which is positive only for $t>\tfrac13$; the first-order coefficients are $3$ and $\tfrac72$, so the inequality holds to first order along this ray.
\end{example}

\paragraph{A repaired composite bound.}
In the shared-perturbation model, put $w_l=\lVert W_l\rVert_2$ and $e_l=\lVert E_l\rVert_2$, and let
\[
\widetilde U^{\base}=\sum_{k=1}^{L}\gamma_k\fro{R_k^{\base}}
\prod_{s=k+1}^{L}\gamma_s(w_s+e_s).
\]
For Lipschitz activations and strengths in $[0,1]$, whenever the corrected recursion is defined,
\[
\fro{\delta_{L+1}^{\qep}}\le\widetilde U^{\base}.
\]
Only the Gram inverses at positive-strength layers are required; the weights need not be invertible. For a fixed finite teacher and input, if $e_l(t)\le t f_l$ with fixed finite $f_l$, then, uniformly over defined corrected trajectories and strengths,
\[
\begin{gathered}0\le\widetilde U^{\base}(t)-U^{\base}(t)\le Ct^2,\\
\fro{\delta_{L+1}^{\qep}(t)}\le U^{\base}(t)+Ct^2.\end{gathered}
\]
for some finite $C$ and $0<t\le t_0$. This is an additive remainder, not an exact comparison with the original bound.

\begin{proof}[Proof of the repaired composite bound]
Write $Z_l=\Xh_l^{\qep}$, $Y_l=\Xh_l^{\base}$,
$B_l=X_l-Y_l$, $D_l=Z_l-Y_l$, and
$b_l=\fro{B_l}$, $d_l=\fro{D_l}$, $h_l=b_l+d_l$.
Then $X_l-Z_l=B_l-D_l$, $h_1=0$ and $\fro{X_l-Z_l}\le h_l$.
Let $P_l$ be the row-space projection of $Z_l$ when $\alpha_l>0$, and set $P_l=0$ otherwise. Both $P_l$ and $I-\alpha_lP_l$ have norm at most one. The exact difference of the pre-activations is
\[
\Wh_l^{\qep}Z_l-\Wh_l^{\base}Y_l
=W_lD_l(I-\alpha_lP_l)+E_lD_l+\alpha_lW_lB_lP_l.
\]
The Lipschitz property therefore gives
\begin{align*}
d_{l+1}&\le\gamma_l[(w_l+e_l)d_l+\alpha_lw_lb_l],\\
b_{l+1}&\le\gamma_l\fro{R_l^{\base}}.
\end{align*}
Since $\alpha_lw_l\le w_l+e_l$, addition yields
\[
h_{l+1}\le\gamma_l\fro{R_l^{\base}}+\gamma_l(w_l+e_l)h_l.
\]
Unrolling from $h_1=0$ proves the claimed bound. The residuals and the projection have stayed on their respective trajectories throughout.

For the remainder, on the uncorrected trajectory
$R_l^{\base}(t)=(W_l+E_l(t))(X_l-Y_l(t))-E_l(t)X_l$, so
\[
\fro{R_l^{\base}(t)}\le(w_l+t f_l)b_l(t)+t f_l\fro{X_l}.
\]
Together with $b_{l+1}(t)\le\gamma_l\fro{R_l^{\base}(t)}$ and $b_1=0$, finite induction gives $b_l(t)=O(t)$ and $\fro{R_l^{\base}(t)}=O(t)$ uniformly. Each downstream product changes by $O(t)$ when its factors $w_s$ are replaced by $w_s+e_s(t)$: expand the finite product, with $0\le e_s(t)\le t f_s$. This does not divide by $w_s$ and also covers zero weights. Multiplying the product differences by the $O(t)$ residuals and summing proves the nonnegative $O(t^2)$ remainder.
\end{proof}

Thus the composite inequality of \cite{qep} is valid to first order in this precise sense, unlike the ordering of the two original bounds. For Example~\ref{ex:composite}, $\widetilde U^{\base}=6$.

\end{document}